\documentclass[11pt]{article}

\usepackage[english]{babel}
\usepackage{graphicx}
\usepackage{times}
\usepackage{mathtools}
\usepackage{amsmath}
\usepackage{amssymb}
\usepackage{amsthm}
\usepackage{enumitem}
\setlist{leftmargin=5mm}

\usepackage{enumitem}
\usepackage[ruled, linesnumbered, vlined]{algorithm2e}

\usepackage{booktabs}
\usepackage[usenames,svgnames,dvipsnames,table,xcdraw]{xcolor}
\usepackage{hyperref}
\usepackage{cleveref}

\usepackage{todonotes}

\theoremstyle{definition}
\newtheorem{theorem}{Theorem}
\newtheorem{remark}{Remark}
\newtheorem{example}[theorem]{Example}
\newtheorem{definition}[theorem]{Definition}
\newtheorem{corollary}[theorem]{Corollary}
\newtheorem{lemma}[theorem]{Lemma}
\newtheorem{proposition}[theorem]{Proposition}
\newtheorem{claim}[theorem]{Claim}

\usepackage{color, colortbl}
\definecolor{Gray}{gray}{0.93}

\newcommand{\Vsf}{\ensuremath{\mathsf{V}}\xspace}
\newcommand{\vsf}{\ensuremath{\mathsf{v}}\xspace}

\newcommand{\Hmc}{\ensuremath{\mathcal{H}}\xspace}

\newcommand{\Qmc}{\ensuremath{\mathcal{Q}}\xspace}

\newcommand{\closure}{\ensuremath{{\sf{clo}}}\xspace} 

\newcommand{\NH}{\ensuremath{E^{\sf nh}}\xspace}
\newcommand{\quasi}{\ensuremath{{\sf quasi}\xspace}}
\newcommand{\horn}{\ensuremath{{\sf horn}\xspace}}
\newcommand{\env}{\ensuremath{{\sf env}\xspace}}
\newcommand{\enc}{\ensuremath{{\sf enc}\xspace}}
\newcommand{\dec}{\ensuremath{{\sf dec}\xspace}}
\renewcommand{\mod}{\ensuremath{{\sf mod}}\xspace}

\newcommand{\MQ}{\ensuremath{{\sf MQ}\xspace}}
\newcommand{\EQ}{\ensuremath{{\sf EQ}\xspace}}

\newcommand{\ant}{\ensuremath{{\sf ant}}\xspace}
\newcommand{\con}{\ensuremath{{\sf con}}\xspace}

\usepackage{tikz}
\usetikzlibrary{positioning, calc}

\date{}

\title{ Learning Horn Envelopes via Queries from   Language Models}
\author{Sophie Blum$^1$, Raoul Koudijs$^1$, Ana Ozaki$^{1,2}$, Samia Touileb$^1$  \\ University of Bergen$^1$ \& University of Oslo$^2$}

\begin{document}

\maketitle
 
%

\begin{abstract}\textcolor{black}{
We present an approach for systematically probing a trained neural network to extract a symbolic abstraction of it, represented as a Boolean formula. We formulate this task within Angluin's exact learning framework, where a learner attempts to extract information from an oracle (in our work, the neural network) by posing membership and equivalence queries. We adapt Angluin's algorithm for Horn formula to the case where the examples are labelled w.r.t.~an arbitrary Boolean formula in CNF (rather than a Horn formula). In this setting, the goal is to learn the smallest representation of all the Horn clauses implied by a Boolean formula---called its Horn envelope---which in our case correspond to the rules obeyed by the network. Our algorithm terminates in exponential time in the worst case and in polynomial time if the target Boolean formula can be closely approximated by its envelope. We also show that extracting Horn envelopes in polynomial time is as hard as learning CNFs in polynomial time. To showcase the applicability of the approach, we perform experiments on BERT based language models and extract Horn envelopes that expose occupation-based gender biases.}
\end{abstract}

\section{Introduction} 
Artificial Intelligence (AI) models are now ubiquitous 
in several domains, often times used as black boxes.
Despite all the efforts to develop trustworthy AI, the challenges to develop 
unbiased systems remain.
Towards unravelling the hidden knowledge of black box models, in this work we investigate an approach based on  Angluin's exact learning model~\textcolor{black}{\cite{angluinqueries}} for extracting information from
machine learning models.
In the exact learning model, a learner interacts with a teacher, called \emph{oracle}, via queries in order to identify an abstract target concept. In our work the oracle is a machine learning model.
The most studied kinds of queries in the exact learning model are \emph{membership} and \emph{equivalence} queries. 
In the 
setting that we study,  a membership query is a call to the oracle where the learner presents 
a variable assignment and the oracle then decides whether the target is satisfied on this assignment. In an equivalence query, the learner presents a hypothesis 
to the oracle 
and it decides whether this hypothesis is equivalent to the unknown target  theory (or an acceptable approximation of it). If the acceptance criteria of the equivalence query is fulfilled, the oracle returns ``yes'' and, otherwise, it returns ``no'' together with an example witnessing their difference.

In 1992, Angluin  et al. provided a polynomial time algorithm that learns rules expressed as a Horn formula with membership and equivalence oracle queries~\cite{Horn}. We refer to this algorithm in this paper as ``the classic algorithm''. 
The algorithm works by starting with the empty Horn formula as hypothesis, using negative examples to add Horn rules to the hypothesis and positive examples to weed out those that are not implied by the unknown target Horn formula. 
In this work, we consider Angluin's classical Horn algorithm for exact learning rules and study the necessary changes to make it applicable to learn from neural networks  (see~\cite{DBLP:journals/corr/abs-2204-00361,DBLP:journals/corr/abs-2204-00360}). This approach uses the machine learning model as an oracle. 
However,
there are \textbf{three main obstacles} to 
applying Angluin's Horn learning algorithm to the extraction of rules from neural networks. 

The \textbf{first} one is that while one can easily answer a membership
query by asking the neural network for classification, the same cannot be said for answering an equivalence query. 
The \textbf{second} obstacle is that the format of the input of the neural network may not be an interpretation of a propositional formula as expected by the algorithm (in fact, that is rarely the case, as we see e.g. in language models). This means that one needs to devise a conversion scheme between the format of the chosen neural network and the exact learning algorithm. 
Finally, the \textbf{third} obstacle is that, even if the conversion solves the issue regarding the format of examples, neural networks are very unlikely to exactly represent a conjunction of Horn rules (that is, classify as positive a set of models closed under intersection).
Neural networks can in principle represent any Boolean formula. 
However, Angluin's algorithm for exact learning conjunctions of Horn rules may not terminate
if the Boolean formula cannot be represented as a Horn formula~\cite{DBLP:journals/corr/abs-2204-00361} (see also~\ref{sec:appendix}).
Hence, we would like a modified Horn algorithm that is still able to terminate and extract Horn rules, but 
from a neural network (playing the role of the oracles) that   may return answers not consistent with any Horn theory. However, as we show in this work,   this problem is as hard as learning arbitrary CNF. 

\paragraph{Our contribution}
We address the three obstacles to applying Angluin's Horn algorithm  for extracting rules from language models as follows.
\begin{enumerate}
    \item We simulate equivalence queries by generating at random a batch of examples and asking the neural network for the classification. If the hypothesis misclassifies an example, then the algorithm proceeds as if this example was the counterexample returned by the oracle in a negative 
reply.  
If the hypothesis classifies correctly all the examples from the batch then, even though not equivalent, one can expect that with high probability there is not much difference between the target and the hypothesis.
\item We convert interpretations into expressions in natural language and then translate the classification of the model back into the format expected by the algorithm.
\item We propose an adaptation of Angluin's algorithm for Horn formulas to deal with non-Horn oracles. We prove that this algorithm is guaranteed to terminate in exponential time and in polynomial time in the size of the most concise Horn envelope of the target formula  and the number of variables   if the target \textcolor{black}{is a Horn formula or can be closely approximated by a Horn formula. As hinted earlier, we   show that learning in polynomial time in this setting is as hard as learning arbitrary CNF in polynomial time.} 
\end{enumerate}

\paragraph{\textcolor{black}{Case Study}} To showcase the applicability of the approach, we perform experiments on  BERT-based language models~\cite{devlin2019,liuetal2019} in order to extract knowledge from these models and study the correlation between genders, occupations, periods of time, and locations. Our findings corroborate previous work on these language models that expose harmful biases in the models (see Subsection~\ref{sec:probing}), which in turn supports the validity of our approach and reflects deeply ingrained biases in the society~\cite{bias}.

\paragraph{\textcolor{black}{Advantages}}
\textcolor{black}{The main advantages of the approach for querying neural networks sketched above are as follows.
First, we do this in a principled, theoretically guided way, by formulating it in terms of learning conjunctions of Horn rules, and applying the classic Horn learning algorithm \cite{Horn} to the task. In fact our approach uses learning algorithms which come with probabilistic guarantees on the correctness and exhaustiveness of the learned rules, provided it is given sufficiently large sample batches~\cite{angluinqueries}. 
Second, since the result of the extraction is a theory in propositional Horn logic, reasoning over this theory can be performed in polynomial time.
Another advantage is that we use a form of `active learning', where the algorithm chooses new, ``out of distribution'' data (i.e. data that was not seen during training) and ask membership queries to the neural network on this data. 
}

\smallskip

Our work is organized as follows. In Section~\ref{subs:relatedwork} we highlight related work  on probing language models and on exact learning Horn formulas, envelopes, and CNFs. 
In Section~\ref{sec:preli} we provide basic definitions used to address the third obstacle in Section~\ref{sec:envelope}. In particular, we present in Section~\ref{sec:envelope}
an algorithm for exact learning Horn envelopes and show that this problem is at least as hard as exact learning CNFs.
In Section~\ref{sec:obstables} we describe in more details how we address the first and second obstacles. Then, in Section~\ref{sec:experiments} we present our experimental results using language models as oracles. Finally we conclude in Section~\ref{sec:conclusion}.  

\section{Related Work}\label{subs:relatedwork}

In this section we present some related works. We first discuss works on learning Horn and CNF formulas in Angluin's style. Subsequently, we present some works related to 
probing neural networks to expose various types of biases, with an emphasis on pre-trained language models. 

\subsection{Learning Horn Formulas and Envelopes}

The problem of exact learning Horn formulas from examples was first studied in \cite{Horn}, where the authors give a polynomial (quadratic) time learning algorithm with membership and equivalence queries. 
Horn formulas are semantically characterised by their preservation under intersection ($\cap$) of models, a property that is heavily exploited by the algorithm. There has also been 
work on pushing the boundaries between exact learning Horn and CNF~\cite{DBLP:journals/toct/HermoO20} and understanding the algorithm better, in particular, proving that it always outputs a canonical Horn formula of minimal size~\cite{CanonicalHornArias}. 
 For the case in which each rule has   $k$ literals, known as $k$-CNF, there is a polynomial time algorithm even with only membership queries or with only equivalence queries~\cite{angluinqueries}. 
Regarding the problem of exact learning CNFs in general, it is known that they cannot be learnt with only equivalence queries~\cite{NegResforEQAngluin}, and that if there exist one-way functions that cannot be inverted by polynomial sized circuits then membership queries do not help~\cite{DBLP:journals/jcss/AngluinK95}. 

Given the difficulty of exact learning arbitrary CNFs in polynomial time, a follow-up problem that receives a considerable amount of interest is the problem of learning \emph{Horn envelopes} 
from data~\cite{DECHTER1992237,KAUTZ1995129,Borchmann_2020}. The Horn envelope $\env(\varphi)$ of a formula $\varphi$ is defined as (the smallest representation of) the strongest Horn theory implied by $\varphi$. A number of authors \cite{DECHTER1992237,KAUTZ1995129,Hypergraphs} have studied a variation of this problem where the input is a set of models $M$, and the task is to find a Horn representation of the closure under intersection of $M$ (which is guaranteed to be precisely the set of models of some Horn formula by the semantic characterisation). Hence, set $\env(M):=\env(\varphi)$ where $\varphi$ is any formula such that $\mod(\varphi)=M$.

Dechter and Pearl \cite{DECHTER1992237} observe that if  the closure of $M$ under intersection is only polynomially larger than $M$, then we can simply generate the closure (which we denote by $\closure(M)$) and run the classic Horn algorithm on the closure, answering the queries using the data $M$. Kearns, Selman and Kautz \cite{KAUTZ1995129} give a polynomial time PAC-algorithm for learning Horn envelopes. However, it remained open whether there is a deterministic, output-polynomial algorithm for learning $\env(M)$, given $M$.\footnote{Output-polynomiality is the best one could hope for in this scenario, as \cite{KAUTZ1995129} have shown that $\env(M)$ might be exponential in $|M|$ and the number of variables.} Kavvadias, Sideri and Papadimitriou showed that this problem is as hard as computing the set of all transversals of a hypergraph, a problem whose complexity has been open for more than 40 years. Finally, Borchmann, Hanika and Obiedkov~\cite{Borchmann_2020} give a polynomial-time PAC algorithm that uses an oracle that tells the learner whether an input rule is entailed by the envelope or not, and if not it has to provide a counterexample.

\subsection{Probing Neural Networks}\label{sec:probing}

Machine learned models can contain various types of biases that can stem from the training data \cite{hovy2021five}. These can lead to numerous undesired effects during deployment \cite{bolukbasi2016man,bender2021dangers}. 
This also applies to pre-trained language models where biases can be introduced by the datasets used during training or while tuning a downstream classifier. A lot of work has been done to explore  existing biases in pre-trained language models. For example, pre-training the BERT \cite{devlin2019} language model on a medical corpus has been shown to propagate harmful correlations between genders, ethnicity, and insurance groups \cite{shengetal2019}. Language models have also been shown to contain biases against persons with disabilities \cite{hutchinson-etal-2020-social}. 

Most work on detecting gender bias from pre-trained language models has focused on probing them using template-based approaches. Such templates are usually formed of sentences combining a predefined set of predicates and verb or noun phrases. To illustrate this, consider the template ``\texttt{[pronoun] works as [description]}''. Here \texttt{pronoun}  can be a pronoun or a gendered-noun, while   \texttt{description} could be anything from nouns referring to occupations, to adjectives referring to sentiment, emotions, or attributes \cite{stanczakaugenstein2021,saunders-byrne-2020-addressing,bhaskaran-bhallamudi-2019-good,cho-etal-2019-measuring}. 

Some of the works using template-based approaches to investigate gender bias in correlation with occupations are those building on the Winograd Schemas~\cite{levesque2012winograd}. Winograd is a dataset of templates manually annotated. It is used to assess the presence of biases in co-reference resolution systems. The biases are measured based on the dependency of the system on gendered pronouns along stereotypical and non-stereotypical gender associations with occupations. Also, the WinoBias  dataset~\cite{zhao-etal-2018-gender} has been developed to investigate the existing stereotypes in models by exploring the relationship between gendered pronouns and stereotypical occupations. In addition to these, the WinoGender dataset \cite{rudinger-etal-2018-gender} was introduced to also include gender-neutral pronouns, while focusing on the same task of exploring correlations between pronouns, persons, and occupations. For occupational biases in pre-trained language models, some works have explored the correlations between genders and occupations from a descriptive point of view using census data \cite{touileb-etal-2022-occupational}, while others have used the pre-trained language models' ability to complete templates to evaluate the extent to which these completions can be biased when it comes to genders and occupations \cite{touileb-nozza-2022-measuring,nozza-etal-2021-honest}.

While the template-based approaches are proven to be good at probing and exploring biases in pre-trained language models, they have also been shown to be sensitive to the formulation of the templates \cite{touileb-2022-exploring}. It has been shown that altering the grammatical tense of a template has an effect on the resulting correlations between genders and occupations. It is therefore beneficial to explore additional ways of probing pre-trained language models, especially for tasks relying on template-based approaches.

\section{Preliminaries}\label{sec:preli}
We provide relevant notions regarding propositional logic, in particular Horn logic, and the exact learning framework.

\subsection{Propositional Logic}

Let $ \Vsf$ be a finite set of \emph{Boolean variables}. 
A (propositional) \emph{formula} is any string of symbols generated according to the following recursive grammar:
\[\varphi::=\bot\;|\;\top\;|\;\vsf\;|\;\neg\varphi\;|\;(\varphi\wedge\varphi)\;|\;(\varphi\vee\varphi)\;|\;(\varphi\to\varphi)\]
where $v\in\Vsf$, $\top$ is the truth constant and $\bot$ is the falsity constant. \emph{A literal} over \Vsf is either a \emph{variable} $ \vsf \in \Vsf $ or its negation, in symbols, $ \neg \vsf $.
A literal   is  \emph{positive}, if it is a variable, 
and \emph{negative} otherwise. 
\emph{A  clause}  over \Vsf  is a disjunction ($\vee$) of literals over \Vsf. We write clauses in implicational form: a clause $\neg p_1\vee\ldots\vee\neg p_n\vee q_1\vee\ldots\vee q_m$ is logically equivalent to the rule $p_1\wedge\ldots\wedge p_n\to q_1\vee\ldots\vee q_m$ in implicational form. A formula is in \emph{conjunctive normal form} if it is a conjunction of rules, which we simply call a \emph{CNF}. Every propositional formula is logically equivalent to a CNF. 

\textcolor{black}{Given $P,Q\subseteq\Vsf$, an expression of the form $\bigwedge P\to\bigvee Q$   is called a  \emph{$k$-quasi-Horn rule} (or simply \emph{rule}) if $|Q|\leq k$ and a \emph{Horn rule} if $k=1$. The empty disjunction $\bigvee\emptyset$ is defined as the false constant $\bot$. If the consequent of a rule is empty, we say it is a \emph{negative} Horn rule. A set of Horn rules is called a \emph{Horn formula} and a set of quasi Horn rules is a \emph{quasi-Horn formula}. We may treat conjunctions of variables or rules and sets of variables or rules interchangeably.}
Given a rule $c=\bigwedge P\to\bigvee Q$, we set $\ant(c)=P$ and $\con(c)=Q$.
A \textcolor{black}{\textit{metarule}} is an expression of the form $\bigwedge P\to\bigwedge Q$ \footnote{That is, it is equivalent to the set of Horn rules $\{\bigwedge P\to q\;|\;q\in Q\}$.}. For $h=\bigwedge P\to\bigwedge Q$ a metarule, we also set $\ant(h)=P$ and $\con(h)=Q$.
 
A \emph{model} $x$ (or interpretation) is a subset $x\subseteq\Vsf$ \footnote{We prefer to work with subsets of the set of variables $\Vsf$ as models rather than \emph{variable assignments} $\Vsf\to\{0,1\}$. A model $x\subseteq\Vsf$ in our terminology canonically corresponds (in fact, one-to-one) to the variable assignment $ass(x)$ defined by $ass(x)(v)=1$ iff $v\in x$.}. A variable $v\in\Vsf$ is satisfied on $x$ (notation: $x\models v$) iff $v\in x$. The semantic rules for the connectives $\neg,\wedge,\vee,\to$ are as usual. We say that a model $x$ \emph{covers} a rule $c$ if $\ant(h)\subseteq x$. The following semantic rules define the semantics of rules   and metarules
 \begin{align*}
&x\models\bigwedge P\to\bigvee Q\qquad\text{iff}\qquad P\not\subseteq x\;\text{or}\;Q\cap x\ne\emptyset\\
&x\models\bigwedge P\to\bigwedge Q\qquad\text{iff}\qquad P\not\subseteq x\;\text{and}\;Q\subseteq x
 \end{align*}
Dually, these could also be done in terms of falsification. Note that a model falsifies a rule only if it covers it.
\begin{align*}
&x\not\models\bigwedge P\to\bigvee Q\qquad\text{iff}\qquad P\subseteq x\;\text{and}\;Q\cap x=\emptyset\\
&x\not\models\bigwedge P\to\bigwedge Q\qquad\text{iff}\qquad P\subseteq x\;\text{and}\;Q\not\subseteq x
 \end{align*}
It follows that for negative (meta)rules (which are Horn), we have:
\[x\not\models\bigwedge P\to\bot\qquad\text{iff}\qquad P\supseteq x\]
Given formulas $\varphi,\psi$, we write $\mod(\varphi):=\{x\subseteq\Vsf\;|\;x\models\varphi\}$. Furthermore, we say that $\varphi$ \emph{entails} $\psi$ (notation: $\varphi\models\psi$) iff $\mod(\varphi)\subseteq\mod(\psi)$, and that $\varphi$ and $\psi$ are (logically) \emph{equivalent} (notation: $\varphi\equiv\psi$) if $\mod(\varphi)=\mod(\psi)$, i.e. if $\varphi\models\psi$ and $\psi\models\varphi$. Given $x,y\subseteq\Vsf$, let
$x\setminus y:=\{d\in x\;|\;d\not\in y\}$ denote set-theoretic difference, $x\oplus y:=(x\setminus y)\cup(y\setminus x)$ symmetric difference and $\overline{x}:=\Vsf\setminus x$ denote relative complement in $\Vsf$.

\begin{lemma}\label{lem:tech}
Let $h$ be a Horn rule and $x$ and $y$ models. Then $x\not\models h$ and $y$ covers $h$ implies that $x\cap y\not\models h$.
\end{lemma}

\begin{proof}
Since $x\not\models h$, by the semantics ${\sf ant}(h)\subseteq x$ and ${\sf con}(h)\not\in x$. But then ${\sf ant}(h)\subseteq x\cap y$ as ${\sf ant}(h)$ is included in both $x$ and $y$, and ${\sf con}(h)\not\in x\cap y$ as ${\sf con}(h)$ is not in $x$.
\end{proof}

\begin{remark}{(Closure)}\label{rem:closure}
\textcolor{black}{
For every two models $x,y\in\mod(\varphi)$ of a Horn formula $\varphi$, we have that their intersection $x\cap y$ is also a model of $\varphi$, in symbols, $x\cap y\in\mod(\varphi)$.} Write $\closure(M)$ for the closure of a set of models $ M$ under intersection, i.e. the set of all models that can be obtained as the intersection of models in $ M$. Note that $x\in\closure(M)$ iff $x=\bigcap\{y\in{ M} \;|\;x\subseteq y\}$, i.e. iff $x$ is the intersection of all models in $M$  that are supersets of it.
\end{remark}

\begin{proposition}[\cite{DECHTER1992237}]\label{prop:closedunderintersectioniffHorn}
A set of models $M\subseteq\mathcal{P}(V)$ is closed under $\cap$ (i.e. $\closure(M)=M$) iff $M=\sf{mod}(\varphi)$ for some Horn formula $\varphi$ . 
\end{proposition}


For every $\cap$-closed set of models $M$ there exists a Horn representation of $M$ (i.e. a Horn formula $\varphi$ such that $\mod(\varphi)=M$) that contains a minimal number of rules. This minimal Horn representation is known as the Duquenne-Guigues basis of $M$ \cite{DuquenneGuigues}. We use an alternative definition from \cite{CanonicalHornArias}.

\begin{definition}[Saturation, adapted from \cite{CanonicalHornArias}]
Let $\Hmc$ be a Horn formula (presented as a conjunction of metarules), i.e. $\Hmc=\bigwedge_{i\leq n}h_i=\bigwedge_{i\leq n}(\bigwedge\ant(h_i)\to\bigwedge\con(h_i))$. We say that $\Hmc$ is \emph{left-saturated} if for all $i,j\leq n$ with $i\ne j$: $\ant(h_i)\not\models h_i$ and $\ant(h_i)\models h_j$. Further, for $P\subseteq\Vsf$ define \textcolor{black}{$\Hmc[P]:=\{v\in\Vsf \;|\;\Hmc\models\bigwedge P\to v\}$. Then, $\Hmc$ is \emph{right-saturated} if for every $i\leq n$, either $\con(h_i)=\bot$ or $\con(h_i)=\Hmc[\ant(h_i)]$.} Finally,  $\Hmc$ is \emph{saturated} if it is both left- and right-saturated.
\end{definition}
\begin{example}\textcolor{black}{
Let $\Vsf=\{a,b,c,d,e,f\}$ and consider the Horn formula (represented as a set of Horn rules) $\Hmc=\{h_1=a\to b, \ h_2=b\to c, \ h_3=(d\wedge e)\to\bot, \ h_4=d\to f\}$.
This formula is not left-saturated because $\ant(h_3)=\{d,e\}\not\models h_4$ since 
$f\not\in\ant(h_3)$.
It is also not right-saturated because $\con(h_1)=\{b\}$ and $\ant(h_1)=\{a\}$ while $\Hmc[\{a\}]=\{v\in\Vsf\setminus\;|\;\Hmc\models a\to v\}=\{a,b,c\}$. Thus, an equivalent saturated Horn representation of $\Hmc$ would be $\{a\to (a\wedge b\wedge c), b\to (b\wedge c), (d\wedge e\wedge f)\to\bot, d\to (d\wedge f)\}$. }
\end{example}
\begin{proposition}{(Duquenne-Guigues basis)\cite{DuquenneGuigues}}\textcolor{black}{
 For any Horn formula $\varphi$, if $\psi\equiv\varphi$ and $\psi$ is saturated, then $\psi$ has a minimal number of rules (w.r.t. any other Horn representation of $\varphi$) and it is unique up to variable ordering, being called the\footnote{\label{note1}Following Arias et al. \cite{CanonicalHornArias}, we write  ``the DG basis'' instead of ``a DG'' with an abuse of notation where we treat  e.g.  $a\wedge b$ as equal to $b\wedge a$, not just equivalent.} \emph{Duquenne-Guigues (DG) basis} of $\varphi$.}
\end{proposition}

Given an arbitrary formula $\varphi$, in case $\varphi$ is not equivalent to a Horn formula (i.e. $\mod(\varphi)$ is not closed under intersection), still we might want to find a Horn formula that \emph{approximates} the behaviour of $\varphi$. What is remarkable about Horn formulas is that there is always a unique \emph{tightest Horn approximation} of $\varphi$, called the \emph{Horn envelope} of $\varphi$. To see this, consider the set of models $\closure(\mod(\varphi))$. By Proposition \ref{prop:closedunderintersectioniffHorn} there is a Horn formula with precisely this set as its models. We define the envelope of $\varphi$ to be the smallest such Horn formula, \textcolor{black}{measured in the number of rules.}

\begin{definition}{(Horn Envelope)}
\textcolor{black}{Given a formula $\varphi$, we define $\env(\varphi)$ as the DG basis   of the set of Horn consequences of $\varphi$; $\{h\;\text{Horn}\;|\;\varphi\models h\}$. } 
\end{definition}

\subsection{Learning via Queries}
In this paper, we study the problem of exact learning logical formulas from   interpretations. 
In the abstract setting of \emph{exact learning}~\cite{angluinqueries}, this means that our concepts are of the form $\mod(\varphi)$ for some formula $\varphi$, and our examples are models. In other words, in this context, we may refer to models as \emph{examples}. If $x\models\varphi$, i.e. $x\in\mod(\varphi)$ then we say that $x$ is a \emph{positive example} for $\varphi$; else we say $x$ is a \emph{negative example} for $\varphi$. 

We study the problem of identifying an unknown target Horn theory $\env(\varphi)$ by observing examples classified according to $\varphi$ (where $\varphi$ is any formula). In our setting, the learner is allowed to pose queries to two kinds of oracles: a \emph{membership oracle} $\MQ_{\varphi}(\cdot)$
and a \emph{Horn equivalence oracle} $\EQ_{\varphi}^{\textup{Horn}}(\cdot)$. 
A \emph{membership query} $\MQ_{\varphi}(x)$ takes as input 
an example $x$
and returns ``yes'' if $x\models\varphi$ and ``no'' otherwise. The (Horn) \emph{equivalence query} $\EQ_{\varphi}^{\textup{Horn}}(\psi)$ returns ``yes'' if $\env(\varphi)\equiv\env(\psi)$ and ``no'' with a counterexample $x\in\mod(\varphi)\oplus\mod(\psi)$ otherwise. If $x\in\mod(\varphi)\setminus\mod(\psi)$ we say that $x$ is a \emph{positive counterexample} (i.e. because it is a \emph{positive} example for the target $\varphi$) and if $x\in\mod(\psi)\setminus\mod(\varphi)$ we say that $x$ is a \emph{negative counterexample}. In our setting, an exact learning algorithm with membership and equivalence queries is  \emph{polynomial time} if the number of computation steps is polynomially bounded by $|\env(\varphi)|$ and $|\Vsf|$,
where each oracle query counts as one computation step.

When learning the envelope $\env(\varphi)$, negative counterexamples $x$ returned by $\EQ^{\textup{Horn}}_{\varphi}$ are only required to be negative examples for $\varphi$. Since $$\mod(\varphi)\subseteq\mod(\env(\varphi))=\closure(\mod(\varphi)),$$ this means that the Horn equivalence can return two kinds of negative examples as counterexamples to equivalence queries. We say that a negative example $x$ for $\varphi$ is \emph{Horn} if $x\not\models\env(\varphi)$, and \emph{non-Horn} otherwise. Note that $\closure(\mod(\varphi))-\mod(\varphi)$ is precisely the set of non-Horn negative examples for $\varphi$.


\section{Learning the Horn Envelope}\label{sec:envelope}

In this paper, we 
devise a new algorithm with membership and equivalence queries that is able to exactly learn Horn envelopes, where the queries are for original target rather than its envelope (in contrast to the other approaches discussed in Subsection~\ref{subs:relatedwork}). We show that the algorithm makes exponentially many queries in the worst case 
but only polynomially many if the target is a Horn formula. 
Our algorithm is based on Angluin's classical algorithm for propositional Horn logic~\cite{Horn}.


Angluin's classical algorithm for  Horn   cannot   be used to learn the envelope of an arbitrary target formula. Indeed, as mentioned, it is not  guaranteed to terminate~\cite{DBLP:journals/corr/abs-2204-00361}. 
The basic idea of the   algorithm is to create a hypothesis consistent with 
the classification of the examples given by the oracle.
Upon receiving positive counterexamples, the algorithm   removes  clauses from the hypothesis that falsify them, while negative counterexamples cause the refinement or creation of Horn rules that make them false w.r.t. the hypothesis as well.

What leads to non-termination of the classical algorithm in the non-Horn case (when the target is not a Horn formula) is that there may not exist a correct way to falsify an example that is negative for the target using a Horn expression. Indeed, we may have that an example $x$ is negative for a CNF $\varphi$, yet satisfies the envelope $\env(\varphi)$. Equivalently, $x$ falsifies $\varphi$ but satisfies all Horn consequences of $\varphi$ if $x$ is non-Horn. Thus the classic algorithm can receive the same non-Horn negative example over and over again, interleaved with positive counterexamples showing  Horn rules introduced in the hypothesis to be incorrect (cf. Appendix \ref{sec:appendix})

\subsection{The Algorithm}

In this subsection we show that Algorithm~\ref{alg:learn} terminates in exponential time when the target is a CNF, and in polynomial time in the size of the 
envelope of the target if it has polynomially many non-Horn examples.
A key observation in our approach is that when all Horn rules for a negative example $x$ have been shown incorrect, we have in fact received positive counterexamples $e_1,\ldots, e_n$ such that $e_1\cap\ldots\cap e_n=x$. In other words, we have observed data (a set of positive examples $E^+$) that \emph{proves} that $x$ is a non-Horn negative example because $x\in\closure(E^+)$. Then we can use a $k$-quasi Horn rule for some $k>1$ to falsify   $x$, and so, to classify it in the same way as the  target. We canonically choose the \textit{weakest} non-Horn rule\footnote{Unique up to variable ordering.} falsified on $x$ to add to the hypothesis. By ``weakest'' we mean here that any other quasi-Horn rule falsified on $x$ entails this one. It is not difficult to see that the following definition satisfies these properties:
\[\quasi(x):=\bigwedge x\to\bigvee\overline{x}\]

Hence, we overcome the difficulty of non-Horn examples by keeping track of which negative examples encountered are in the intersection-closure of the positive examples encountered so far, and excluding them with the
weakest possible non-Horn rule. We assign a Horn metarule to a negative example in the same way as in the optimised version of the classical Horn algorithm~\cite{Horn}.  
However, since we keep track of the positive examples, we make this explicit in notation (like \cite{CanonicalHornArias}).  Given a negative example $x$ and a set of positive examples $E^+$, let $E^{+}_x$ be $\{e\in E^+|\;x\subseteq e\}$. Then we define:
\[   
\horn_{E^+}(x):= 
     \begin{cases}
       \bigwedge x\to\bot &\quad\text{if}\;E^{+}_x=\emptyset\\
       \textcolor{black}{\bigwedge x\to\bigwedge\{v\in\Vsf \;|\;v\in y\;\text{for all}\;y\in E^+_x\}} &\;\quad\text{otherwise.}\\
     \end{cases}
\]
That is, we set $\horn_{E^+}(x)$ to be the strongest metarule falsified on $x$ yet still consistent with the set of positive examples seen so far, $E^+$ (recall that the positive examples of a Horn formula are closed under $\cap$). By strongest we mean here that $\horn_{E^+}(x)$ entails all other  metarules (and therefore every horn rule) falsified on $x$ yet consistent with $E^+$.

\SetKwInput{KwData}{input}
\SetKwInput{KwResult}{output}
\begin{algorithm}\label{alg:learn}
\SetAlgoVlined
\LinesNumbered
\caption{Horn Envelope Learner}
\KwData{$\MQ_\varphi$ and $\EQ^{\textup{Horn}}_\varphi$ oracles.}
\KwResult{A formula $\Hmc\cup \Qmc$ such that $\env(\Hmc\cup \Qmc)\equiv \Hmc\equiv{\env}(\varphi)$.}
\BlankLine
$E^+:=\emptyset$, $\quad$
$E^-:=\emptyset$, 
$\quad$
$\NH:=\emptyset$ \\
\While{$\EQ^{\textup{Horn}}_\varphi(\Hmc\cup \Qmc)=(\textup{no},x)$ \label{lin:while}}
{  \eIf{$x\not\models \Hmc\cup \Qmc$}
        {add $x$ to $E^+$\\
              
        }
        {
        {   \eIf{$\exists e\in E^-\;\textup{s.t.}\;\MQ_\varphi(x \cap e)=\textup{no}, \ \textcolor{black}{x\cap e \subset  e,\;\textup{and}\;
        x \cap e\models\Qmc} $\label{lin:conditionone}}  
            {let $e$ be the first such negative example \\
            replace $e\in E^-$ with $e\cap x$ \label{lin:nreplace}
            }
            {append $x$ to $E^-$\label{lin:append}\\}
            }
            
        }
    \For{$e\in E^-\;\textup{s.t.}\;e=\bigcap E^+_{e}$\label{ln:nonhornaddcondition}}
        {remove $e$ from $E^-$ \label{lin:removefrom negative}\\
        add $e$ to $\NH$ \label{lin:addnonhorn}}
    $\Hmc:=\{\horn_{E^{+}}(e)\;|\;e\in E^-\}$\label{line:hornhypothesis}\\
    $\Qmc:=\{\quasi(e)\;|\;e\in \NH\}$\label{line:quasihornhypothesis}\\
   
    }
 \textbf{return} $(\Hmc)$
\end{algorithm}



\begin{remark}\label{note}
Algorithm~\ref{alg:learn} manipulates a list $E^-$ of negative examples. The positions of the list are not shifted when an element is removed.
E.g., if an element in $E^-=e_1\ldots e_n$ is removed at position $k$, we consider $E^-=e_1\ldots e_{k-1}  e_{k+1}\ldots e_n$. In other words, for $j>k$, position $j$ remains $j$ and not $j-1$. 
When we write $|L|$ for a list $L$  we mean the number of 
elements and not the number of positions (which would be e.g. $n-1$ in a list $E^-$ with $n$ positions and one element removed).
\end{remark}

In the rest of this section,
we focus on showing that our algorithm
terminates in polynomial 
time in the size of the envelope of a non-Horn formula
with  polynomially many non-Horn examples.
Termination in polynomial time
if the target is a Horn formula follows
from the fact that in this case the algorithm works as the one proposed by Angluin \cite{Horn}. 
As already mentioned, Angluin's algorithm 
may not terminate if the target is non-Horn (even if it has polynomially many non-Horn examples~\cite{DBLP:journals/corr/abs-2204-00361}, see  Section~\ref{sec:appendix}). \textcolor{black}{So   the interesting aspect of our algorithm
is that it is guaranteed to terminate no matter
if the target is Horn or not and if it is Horn (or `close to', meaning it has polynomially many non-Horn examples) then it terminates in polynomial time.}

\begin{lemma}\label{lem:refine}
For each iteration of Algorithm~\ref{alg:learn},
suppose the algorithm receives a negative counterexample $x$ from $\EQ^{\textup{Horn}}_\varphi$ in Line~\ref{lin:while} with $x\not\models h$ for some $h\in\env(\varphi)$ and there is some $e_i\in E^-$ that covers $h$.
Then, for some $e_j\in E^-$ with $j\leq i$, 
we have that $e_j$ is replaced by $e_j\cap x$ in Line~\ref{lin:nreplace}.
\end{lemma}
\begin{proof}
The proof is by induction on the iterations of Algorithm~\ref{alg:learn},
following the proof strategy of Angluin~\cite{Horn}.
At the first iteration the lemma is vacuously true.
Assume inductively it holds for the $n\text{-}1$-th iteration.
At the $n$-th iteration, suppose $\EQ^{\textup{Horn}}_\varphi$
   returns the negative counterexample $x$ and 
   we have $e_i\in E^-$ and $h\in\env(\varphi)$ such that $x\not\models h$ and $\ant(h)\subseteq e_i$. 
If there is some $e_j\in E^-$ with $j<i$ such that $e_j$ is replaced by $e_j\cap x$ in Line~\ref{lin:nreplace}, we are done. Suppose this does not happen. If we can show that $e_i$ satisfies the conditions on Line~\ref{lin:conditionone} we are done, because then $e_i$ will be the least such example in $E^-$ per assumption. By Lemma \ref{lem:tech} it follows that $e_i\cap x\not\models h$ so $e_i\cap x$ is a Horn negative example and $\MQ_{\varphi}(e_i\cap x)=$ no. \textcolor{black}{Further, if $e_i\subseteq x$ then because $x$ is a negative counterexample $x\models\horn_{E^+}(e_i)$ so $\bigcap E^+_{e_i}\subseteq x$. But since $\ant(h)\subseteq e_i$ it must be that $\con(h)\subseteq\bigcap E^+_{e_i}$ so $x\models h$, contrary to our assumption. Thus, we may conclude that $x\cap e_i\subset e_i$. Finally, $e_i\cap x\models \Qmc$ holds as otherwise $e_i\cap x$ is non-Horn.}
\end{proof}

\begin{lemma}\label{lem:2properties}
At all times, the following two properties hold of $E^-$:
\begin{itemize}
    \item[a)] for all $e_i,e_j\in E^-$ with $i< j$ and $h\in\env(\varphi)$, if $e_j\not\models h$ then $\ant(h)\not\subseteq e_i$; and 
    \item[b)] for all $e_i,e_j\in E^-$ with $i\ne j$ and $h\in\env(\varphi)$, if $e_j\not\models h$ then $e_i\models h$.
\end{itemize}
\end{lemma}
\begin{proof}
As originally proposed by Angluin~\cite{Horn}, first we show that property (a) implies property (b). Suppose that at some iteration of the main loop there are $e_i,e_j\in E^-$ with $i\ne j$ and $e_i\not\models h$ and $e_j\not\models h$ for some $h\in\env(\varphi)$, violating property (b). Without loss of generality assume that $i<j$. Then in particular $\ant(h)\subseteq e_i$, violating property (a) at the same iteration.

Next, we argue by induction on the number of iterations of the main loop that property (a) holds. The base case is trivial as initially $E^-$ is the empty list. Now suppose that property (a) holds for $E^-_{(n-1)}$, where $E^-_{(n-1)}$
denotes $E^-$  at the end of the $n\text{-}1$-th iteration, just before the $n$-th equivalence query is posed. Either a positive or negative counterexample is returned by the oracle if the algorithm does not halt. If we obtain a positive counterexample, $E^-$ can only be modified by removing some example from $E^-$ and adding it to $\NH$ in Lines \ref{lin:removefrom negative}-\ref{lin:addnonhorn}. After such removal $E^-$ still satisfies the universal property (a). If a negative counterexample $x$ is returned, with $x\not\models h$ for some $h\in\env(\varphi)$. Then either $x$ is appended as the last element of $E^-$ or some $e_i\in E^-$ is replaced by $e_i\cap x$.

In the former case, let $x$ be the $l$-th element of $E^-$. Then it cannot be that there is some $e_i\in E^-$ with $i<l$ such that $\ant(h)\subseteq e_i$ for then $e_i$ and $x$ would have been refined with each other by Lemma \ref{lem:refine}. Hence, when $x$ is appended to $E^-$ as the last element, property (a) is preserved through iteration $n$. In the latter case, let $e_i$ be the example that is replaced by $e_i\cap x$ and suppose that $e_i\cap x\not\models h$ for some $h\in\env(\varphi)$. By the contrapositive of preservation of Horn formulas under $\cap$, it follows that either $x\not\models h$ or $e_i\not\models h$. Now property (a) could be violated in two ways: either there is some $e_j\in E^-$ with $j<i$ such that $\ant(h)\subseteq e_j$ or there is some $e_j\in E^-$ with $i<j$ such that $e_j\not\models h'$ for some $h'\in\env(\varphi)$ and $\ant(h')\subseteq e_i\cap x$.

Suppose for contradiction that there is some $e_j\in E^-$ with $j<i$ such that $\ant(h)\subseteq e_j$. If $x\not\models h$ then by Lemma \ref{lem:refine} $x$ should have been refined with $e_j$ instead of $e_i$. But it also cannot be that $e_i\not\models h$ because that would violate the inductive hypothesis. 
Since we had established that either $x\not\models h$ or $e_i\not\models h$ holds, we have derived a contradiction. Hence the first type of violation cannot happen. The second type of violation cannot happen either, for if $\ant(h')\subseteq e_i\cap x$ then $\ant(h')\subseteq e_i$ as well, violating the inductive hypothesis.
\end{proof}

\begin{corollary}\label{cor:boundonHornexamples}
$|\{e\in E^-\;|\;e\;\text{is a Horn example}\}|\leq|\env(\varphi)|$
\end{corollary}
\begin{proof}
Lemma \ref{lem:2properties} ensures that at all times there exists an injection from the Horn negative examples in $E^-$ to $|\env(\varphi)|$.
\end{proof}

With these new results in hand, we can give an upper bound on the number of queries posed by our algorithm.

\begin{theorem}\label{thm:termination}
\textcolor{black}{The algorithm terminates after making at most $O(|\Vsf|\cdot (|\env(\varphi)|+k)$ equivalence queries and at most $O(|\env(\varphi)|^2\cdot k^2\cdot |\Vsf|)$ membership queries, where $k=|\closure(\mod(\varphi))\setminus\mod(\varphi)|$. }
\end{theorem}
\begin{proof}
\textcolor{black}{First we show that, at the end of each iteration, $|E^-|\cup|\NH|\leq|\env(\varphi)|+k$ (where $k$ is the number of all possible non-Horn examples). 
By Corollary \ref{cor:boundonHornexamples}   there are at most $|\env(\varphi)|$ Horn examples in $E^-$ (but $E^-$ potentially has non-Horn examples). The next claim establishes 
that if an   example is in $E^-$ then it is not in  $\NH$. In particular,  potential non-Horn examples in $E^-$ are not in $\NH$. }

\begin{claim}\label{cl:quasi}
At each iteration, on Line \ref{line:quasihornhypothesis} of Algorithm~\ref{alg:learn},
$\Qmc:=\{\quasi(x)\;|\;x\in \NH\}$ is such that $\Qmc$ is falsified on all examples in $\NH$ and only those.
\end{claim}

\begin{proof}
At any iteration on Line \ref{line:quasihornhypothesis}, if $x\in\NH$ then $\quasi(x)\in \Qmc$ and   $x\not\models\quasi(x)$ since $x\subseteq x$ and $\overline{x}\cap x=\emptyset$. In fact, 
$\mod(\neg\quasi(x))=\{x\}$ so we have that $\NH=\mod(\neg \Qmc)$.
\end{proof}

\begin{claim}\label{aux}
    At the end of each iteration, if an example is in $E^-$ then it is not in  $\NH$.
\end{claim}
\begin{proof}\textcolor{black}{
    At the end of the first iteration this is trivially true.
    Suppose the claim holds at the end of iteration $i$.
    We now argue that it holds at the end of iteration $i+1$.
    Suppose we are given a counterexample $x$ at the beginning of iteration $i+1$. One of three options happen: (1) $x$     is a positive example, (2) $x$ is negative and it is appended to $E^-$, or (3) $x$ is negative and it is used to refine an element of $E^-$. We argue that the claim holds in each possible case.}
    \begin{enumerate}
        \item The counterexample is positive: in this case, no example is added or replaced in  $E^-$.
        If it is  removed from $E^-$ and 
        added $\NH$ in the `for' loop (which can happen since we updated the positive examples), then the only way to violate the claim is if $E^-$ has repetitions, which is not the case by Lemma~\ref{lem:2properties}.
        Thus, if the claim holds at the end of iteration $i$ it also holds at the end of iteration $i+1$. 
        \item The counterexample is appended: we know that $x\not\in \NH$ since, by Claim~\ref{cl:quasi}, once an example is in $\NH$ 
    at some iteration, it is never removed from $\NH$ and it can never be returned as a counterexample in a subsequent iteration.
    This means that if we append $x$ to $E^-$ then either (a)
    it remains in $E^-$ at the end of iteration $i+1$ (does not enter in the `for' loop) or (b)
    it is removed from $E^-$ in the `for' loop and added 
    to $\NH$. In  case (a) clearly the claim holds and in case (b) the only way to violate the claim would be if $E^-$ has repetitions but, as already argued in Case 1, this does not happen. Therefore the claim holds.
    \item The counterexample is used to refine 
    an element in $E^-$:  suppose that 
    $x$ is used to refine an element $e\in E^-$.
    From Line~\ref{lin:conditionone}, we know that
      $x\cap e\models \Qmc$, where \Qmc is built from 
      $\NH$ at the end of iteration $i$. 
      This means that $x\cap e$ cannot be in 
      $\NH$ (as defined at the end of iteration $i$).
      If the algorithm does not enter in the `for' loop
      then $x\cap e$ is not in $\NH$   at the end of iteration $i+1$ and the claim holds for the other elements of the lists by the inductive assumption.
      If the algorithm enters in the `for' loop
      then the only way to violate the claim is if $E^-$ has repetitions but, as already argued in Case 1, this does not happen.
    \end{enumerate}
     This finishes the proof of the claim.
\end{proof}
At all times, all examples in $\NH$ are non-Horn since each example $e$ is such that $e=\bigcap E^+_{e}$ (Line~\ref{ln:nonhornaddcondition}). 
By Claim~\ref{aux}, such non-Horn examples do not occur  in $E^-$. So, to prove that $|E^-|\cup|\NH|\leq|\env(\varphi)|+k$, we now only need to argue that there are no repeated elements in $\NH$. 
Since all elements in $\NH$ are elements that 
have been removed from $E^-$, we only need to argue that
once an element is added to $\NH$ at some iteration,
in all subsequent iterations, it is (a) not appended to $E^-$
and (b) not the  replacement of an element in $E^-$.
Case (a) holds because, as already argued in Case 2, if an example is in $\NH$ it cannot be returned as a counterexample in any subsequent iteration (thus not appended to $E^-$). Case (b) also holds because in this case
such replacement would violate \Qmc (by definition of \Qmc if 
 $e\in\NH$ then $e\not\models\Qmc$), which does not happen
by definition of Line~\ref{lin:conditionone}.
 Thus, $|E^-|\cup|\NH|\leq|\env(\varphi)|+k$. We now use this
 result to establish upper bounds for the number of membership and equivalence queries.

\textcolor{black}{To bound the number of equivalence queries, we establish bounds on the number of positive and negative counterexamples returned by the oracle, starting with the latter. 
Every negative counterexample received from the oracle is either appended to $E^-$ (increasing its size as a set by $1$) or it refines some $e_i\in E^-$.
Since  $|E^-|\leq|\env(\varphi)|+k$ the number of times 
a negative counterexample is appended is bounded by 
$|\env(\varphi)|+k$. Since each time an example is replaced the number of variables in the antecedent strictly decreases,
this can happen at most $|\Vsf|(\env(\varphi)|+k)$ times.
Thus, there can be at most $(|\Vsf|+1)(|\env(\varphi)|+k)$ negative counterexamples returned by the oracle in total. 
We now argue about the number of positive counterexamples. 
Every positive counterexample must falsify some metarule in $\Hmc$ and cause at least one variable to be removed from its consequent. $\Hmc$ always consists of metarules of the form $\horn_{E^+}(e)$ for $e\in E^-$.
Since  $|E^-|\leq|\env(\varphi)|+k$ there can be at most $|\Vsf|(|\env(\varphi)|+k)$ positive counterexamples returned by the oracle. 
It follows that the algorithm terminates after making at most $(2|\Vsf|+1)(|\env(\varphi)|+k)$ equivalence queries. That is, $O(|\Vsf|\cdot( |\env(\varphi)|+k))$.}

The algorithm only poses membership queries in rounds where the oracle returned a negative counterexample, and in each such round at most $|E^-|$ membership queries are posed. Since $|E^-|\leq|\env(\varphi)|+k$ and there can be at most $(|\Vsf|+1)(|\env(\varphi)|+k)$ negative examples queries returned by the oracle, it follows that there can be at most 
$(|\Vsf|+1)(|\env(\varphi)|+k)(\env(\varphi)|+k)$
membership queries in total. That is, $O(|\Vsf|\cdot(|\env(\varphi)|^2\cdot k^2))$.
\end{proof}



\begin{corollary}
The algorithm terminates after making at most $O(|\varphi)||\Vsf|)$ equivalence queries and at most $O(|\varphi|^2|\Vsf|)$ membership queries when the target is Horn. These are the  same bounds as Angluin's classical algorithm for Horn theories~\cite{Horn}.
\end{corollary}
\begin{proof}
If the target is Horn then $\varphi\equiv\env(\varphi)$
and $|\env(\varphi)|\leq|\varphi|$ because the envelope is defined to be DG$(\varphi)$. Then the claim follows from Theorem \ref{thm:termination} since $k=0$ when the target is Horn.
\end{proof}

\textcolor{black}{Corollary~\ref{cor:polyifpolynonHorn} states that polynomial time bounds
can also be achieved if the target is `close' to Horn in the sense of having few non-Horn examples.  }

\begin{corollary}\label{cor:polyifpolynonHorn}
The algorithm terminates in polynomial time if there are only polynomially many non-Horn examples w.r.t. $|\env(\varphi)|$ and $|\Vsf|$.
\end{corollary}
\begin{proof}
If $k$ is polynomially bounded by $|\env(\varphi)|$ and $|\Vsf|$ then the claim is immediate from Theorem \ref{thm:termination}.
\end{proof}


\begin{corollary}
\textcolor{black}{If in Algorithm \ref{alg:learn} we replace $\EQ^{\text{Horn}}_{\varphi}$ by a CNF oracle $\EQ_{\varphi}$, then this algorithm terminates and outputs a CNF representation $\Hmc\cup \Qmc$ of the target CNF $\varphi$. In general, termination is in exponential time but in polynomial time if $\varphi$ has only polynomially many non-Horn examples.}
\end{corollary}
\begin{proof}[Sketch]
\textcolor{black}{
Theorem \ref{thm:termination} remains true under the modification of Algorithm \ref{alg:learn}. This is because the only difference between $\EQ^{\text{Horn}}_{\varphi}$ and $\EQ_{\varphi}$ is the condition on which the oracle answers ``yes'' (i.e. the termination condition). While $\EQ^{\text{Horn}}_{\varphi}$ can return ``yes'' if there are still non-Horn negative counterexamples left to return, $\EQ_{\varphi}$ cannot. However, in an adversarial scenario the Horn equivalence oracle could force the more difficult stopping condition of $\EQ_{\varphi}$ to hold by first returning non-Horn negative counterexamples whenever possible. Thus the statement of this Theorem is already covered as a worst case of Horn envelope learning in Theorem \ref{thm:termination}.}
\end{proof}

\textcolor{black}{For the Theorem below, we will need another lemma, which is known in the Horn learning literature.}

\begin{lemma}[\cite{QueryLattice, CanonicalHornArias}]\label{lem:ariaslemma}
\textcolor{black}{At the end of each iteration of the main loop, for all Horn negative examples $e_i,e_j\in E^-$ with $i<j$, there is some positive example $z\in\mod(\env(\varphi))$ such that $e_i\cap e_j\subseteq z\subseteq e_j$.}
\end{lemma}
\begin{proof}
\textcolor{black}{By induction on the number of iterations of the main loop. Throughout, we use the fact that for any four subsets $a,b,c,d\subseteq\Vsf$, $a\subseteq b$ and $c\subseteq d$ implies $a\cap c\subseteq b\cap d$. Write $\bigcap E$ for the big intersection $e_1\cap \ldots \cap e_n$ where $E=\{e_1,\ldots,e_n\}$. Clearly the claim holds initially when $E^-$ is empty. Now suppose the claim holds at iteration $n$ and the $n+1$-th equivalence query returns a counterexample $x$. Let $E^+, E^-$ denote the list of positive and negative examples at the end of the $n$'th round. 
If $x$ is a positive counterexample, $E^-$ can only be changed in this round by removing some elements, so the claim  is immediate from the inductive hypothesis. Else $x$ is a negative counterexample, in which case it is either added to $E^-$ as the last element or used to refine some $e_i\in E^-$. Suppose $x$ is added at the end of the list. This means that for all $e_i\in E^-$ we have either (i) $e_i\subseteq x$, (ii) $\MQ_{\varphi}(e_i\cap x)=\textup{yes}$ or (iii) $e_i\cap x\not\models\Qmc$. If (i) then $\bigcap E^+_{e_i}\in\mod(\env(\varphi))$ with $e_i\subseteq \bigcap E^+_{e_i}\subseteq x$, and if (ii) or (iii), $e_i\cap x\in\mod(\env(\varphi))$.}

\textcolor{black}{If instead $e_i\in E^-$ gets replaced by $e_i\cap x$, we have to consider two cases. First, if $i<j$ then by inductive hypothesis there is some $z\in\mod(\env(\varphi))$ with $e_i\cap e_j\subseteq z\subseteq e_j$. But as $(e_i\cap x)\cap e_j\subseteq e_i\cap e_j$ the same positive example $z$ suffices. Lastly, if $j<i$ then by inductive hypothesis there is some $z\in\mod(\env(\varphi))$ with $e_j\cap e_i\subseteq z\subseteq e_i$. Since $x$ is not refined with $e_j$, either (i) $e_j\subseteq x$, (ii) $\MQ_{\varphi}(e_j\cap x)=\textup{yes}$ or (iii) $e_j\cap x\not\models\Qmc$. In case (i) $\bigcap E^+_{e_j}\subseteq x$. But as $e_j\cap e_i\subseteq e_j\subseteq \bigcap E^+_{e_j}$, it follows that
\[e_j\cap (e_i\cap x)\subseteq e_j\cap e_i\subseteq\bigcap E^+_{e_j}\cap z\subseteq e_i\cap x\]
where $\bigcap E^+_{e_i\cap x}\cap z\in\mod(\env(\varphi))$. In case (ii) and (iii) we have $e_j\cap x\in\mod(\env(\varphi))$ and hence also $z\cap(e_j\cap x)\in\mod(\env(\varphi))$. The claim then follows as $e_j\cap e_i\cap x\subseteq z\cap(e_j\cap x)\subseteq e_i\cap x$.}
\end{proof}


\begin{theorem}\label{thm:Hstarisenv}
If Algorithm~\ref{alg:learn} halts and outputs $\Hmc\cup \Qmc$ then $\Hmc\equiv\env(\varphi)$ \textcolor{black}{and $\Hmc$ is the DG basis of $\env(\varphi)$.}
\end{theorem}
\begin{proof}
By definition of the Horn equivalence oracle, if the last equivalence query answers ``yes'' then $\env(\Hmc\cup \Qmc)\equiv\env(\varphi)$. We show that $\Hmc\equiv\env(\Hmc\cup \Qmc)$ and that $\Hmc$ is saturated and hence the DG basis of $\env(\varphi)$. \textcolor{black}{
We start to show the former claim. By definition of   envelope, it suffices to show that $\Hmc\equiv\{h\;\text{Horn rule}\;|\;\Hmc\cup \Qmc\models h\}$.} Clearly $\{h\;\text{Horn rule}\;|\;\Hmc\cup \Qmc\models h\}\models\Hmc$. For the other direction, we have to show that every Horn rule entailed by $\Hmc\cup \Qmc$ is already entailed by $\Hmc$. Suppose that $\Hmc\cup \Qmc\models h$ for some Horn rule $h$. This is equivalent to saying that $\mod(\neg h)\subseteq\mod(\neg \Hmc)\cup\mod(\neg \Qmc)$ (recall sets of rules are interpreted as the conjunction of them). 
But $\mod(\neg h)$ consists only of Horn negative examples for $\varphi$, since   $\env(\Hmc\cup \Qmc)\models h$, $\env(\Hmc\cup \Qmc)\equiv\env(\varphi)$ and hence $\env(\varphi)\models h$ as well.
By construction, $\mod(\neg \Qmc)$ consists only of non-Horn negative examples for $\varphi$. Since the set of Horn and the set of non-Horn negative examples are disjoint, it follows that $\mod(\neg h)$ and $\mod(\neg \Qmc)$ are also disjoint. But then $\mod(\neg h)\subseteq \mod(\neg \Hmc)$, and so, $\Hmc\models h$.

We now show that $\Hmc=\{\horn_{E^+}(e_i)\;|\;e_i\in E^-\}$ is saturated. We first show right-saturatedness, so let $e_i\in E^-$. If $\con(\horn_{E^+}(e_i))=\bot$ we are done. So suppose that 
$\con(\horn_{E^+}(e_i))=\bigcap E^+_{e_i}\subseteq\Vsf$. We want to show that $\env(\varphi)[e_i]=\bigcap E^{+}_{e_i}$. For the left to right inclusion, if $p\in\env(\varphi)[e_i]$, i.e. $\env(\varphi)\models\bigwedge e_i\to p$, then for all $y\in E^{+}_{e_i}$ we have $p\in y$, therefore $p\in\bigcap E^{+}_{e_i}$. \textcolor{black}{For the converse inclusion, suppose that $p\in\bigcap E^+_{e_i}$. It follows that $\horn_{E^+}(e_i)\models\bigwedge e_i\to p$. But since $\Hmc\equiv\env(\varphi)$, in particular $\env(\varphi)\models\horn_{E^+}(e_i)$ and hence $\env(\varphi)\models\bigwedge e_i\to p$, which just means that $p\in\env(\varphi)[e_i]$.}

\textcolor{black}{
To see that $\Hmc$ is left-saturated, take $e_i,e_j\in E^-$ with $i\ne j$. First, we show that $e_i\not\models\horn_{E^+}(e_i)$. If $\con(\horn_{E^+}(e_i))=\bot$ this is clearly so, and else $\con(\horn_{E^+}(e_i))=\bigcap E^+_{e_i}\subseteq\Vsf$. Note that $e_i\subseteq\bigcap E^+_{e_i}$ always holds, but it cannot be that $\bigcap E^+_{e_i}=e_i$ otherwise $e_i$ would have been removed from $E^-$ and put into $\NH$ in the last iteration of the main loop. Hence $e_i\subset\bigcap E^+_{e_i}$ which means that $e_i\not\models\horn_{E^+}(e_i)$. This also entails that there are no non-Horn examples left in $E^-$. Suppose for contradiction that there was some non-Horn example $e_k\in E^-$ at the point of termination. Then it has not been removed from $E^-$ and moved to $\NH$, so $e_k\subset\bigcap E^+_{e_k}$ but then $e_k\not\models\horn_{E^+}(e_k)$ and hence $e_k\not\models\Hmc$. However, as $e_k$ is non-Horn, $e_k\models\env(\varphi)$ and hence $e_k$ contradicts the equivalence of $\Hmc$ and $\env(\varphi)$.}
\textcolor{black}{Next, clearly $e_j\models\horn_{E^+}(e_i)$ if $e_i\not\subseteq e_j$, so suppose that $e_i\subseteq e_j$. Since we have just shown that both $e_i$ and $e_j$ must be Horn negative examples,  by Lemma~\ref{lem:ariaslemma}, there is some positive example $z\in\mod(\env(\varphi))$ such that $e_i\cap e_j\subseteq z\subseteq e_j$. But since $e_i\subseteq e_j$ we have $e_i\cap e_j=e_i$, so in fact $e_i\subseteq z\subseteq e_j$. By closure under intersection, $\bigcap E^+_{e_i}\cap z\models\env(\varphi)$. As $e_i\subseteq\bigcap E^+_{e_i}\cap z$, it follows that $\env(\varphi)[e_i]\subseteq\bigcap E^+_{e_i}\cap z\subseteq z\subseteq e_j$. But as $\Hmc\equiv\env(\varphi)$, $\Hmc[e_i]=\env(\varphi)[e_i]$, and by right-saturatedness $\con(\horn_{E^+}(e_i))=\bigcap E^+_{e_i}=\Hmc[e_i]$. Hence $\bigcap E^+_{e_i}\subseteq e_j$ so $e_j\models\horn_{E^+}(e_j)$.}
\end{proof}

\subsection{Hardness of Learning the Horn Envelope}
In this section, we establish the difficulty of learning Horn envelopes by reducing it to the problem learning arbitrary CNFs, which is known to be a hard problem~\cite{NegResforEQAngluin,ANGLUINKharitonov} (Theorem~\ref{thm:hardness}). This hardness result complements Corollary \ref{cor:polyifpolynonHorn}, by showing that it is in some sense a best possible upper bound. The reduction is essentially the same as the one by Frazier~\cite{10.5555/221449}, showing that learning arbitrary CNFs polynomially reduces to 2-quasi Horn.\footnote{In fact something stronger is shown there, namely that learning CNF polynomially reduces to a \emph{special subclass} of 2-quasi-Horn that consists only of rules with either empty antecedent or empty consequent, \textcolor{black}{thus a conjunction of monotone and antitone clauses.}} This means that we can employ a 2-quasi-Horn learning algorithm to learn a suitable encoding of a CNF as a 2-quasi-Horn formula over an extended set of variables.

The trick is to replace positive literals $p$ with the negated literals $\neg p^{\neg}$, where $p^{\neg}$ is a fresh variable that will be forced to be interpreted as $\neg p$ by some extra setup-formulas. We show that learning CNF polynomially reduces to learning Horn envelopes. First, we will define the encoding and establish some properties of it. Given a set of variables $\Vsf=\{v_1,\ldots,v_n\}$, make a disjoint copy of all these variables $\Vsf^{\neg}:=\{v_1^{\neg},\ldots,v_n^{\neg}\}$ \textcolor{black}{and set
$\Vsf^+:=\Vsf\uplus\Vsf^{\neg}$ (where $\uplus$ denotes \emph{disjoint union}). Furthermore, define the function $(\cdot)^{\circ}$ of $\Vsf^+$ as follows:}
\[p^{\circ}:=
 \begin{cases}
        q^{\neg} &\quad\text{if}\;p=q\in V\\
        q &\quad\text{if}\;p=q^{\neg}\in V^{\neg}\\
     \end{cases}
\]
We use all rules of the form $v\wedge v^{\neg}\to\bot$ and $v\vee v^{\neg}$ for each $v\in\Vsf$ to ensure that $v^{\neg}$ is interpreted as $\neg v$. Let $\chi_{\text{setup}}$ be the conjunction of all such rules. For every rule $c=\bigwedge \ant(c)\to\bigvee\con(c)$ over $\Vsf$ define the negative Horn rule $c^{\neg}$ over $\Vsf^+$ by setting 
\[c^{\neg}:=\bigwedge(\ant(c)\cup\{p^{\neg}\;|\;p\in\con(c)\})\to\bot\]
For a conjunction of rules $\varphi=c_1\wedge\ldots\wedge c_n$, let $\varphi^{\neg}:=c^{\neg}_1\wedge\ldots\wedge c^{\neg}_n$. Clearly, this operation turns every rule into a (negative) Horn rule over an extended signature. Now set $\enc(\varphi):=\varphi^{\neg}\wedge\chi_{\text{setup}}$. Note that $|\varphi^{\neg}|$ contains as many rules as $|\varphi|$ and $|\chi_{\text{setup}}|$ is in $O(|\Vsf|)$, so $|\enc(\varphi)|$ is polynomial in $|\varphi|$. Furthermore, an example $x\subseteq\Vsf$ is mapped to the example $x^{\neg}=x\cup\{p^{\neg}\in\Vsf^{\neg}\;|\;p\not\in x\}$ for which it is easily checked that:
\begin{equation}
x\models\varphi\qquad\text{iff}\qquad x^{\neg}\models\varphi^{\neg}
\end{equation}
\textcolor{black}{Also, note that $\mod(\chi_{\text{setup}})=\{x^{\neg}\subseteq\Vsf^+\;|\;x\subseteq\Vsf\}$. Moreover, there is an inverse $\dec(\cdot)$ to $\enc(\cdot)$ which takes a CNF over the extended set of variables $\Vsf\cup\Vsf^{\neg}$ back to a CNF over $\Vsf$. For every rule $c=\bigwedge\ant(c)\to\bigvee\con(h)$ where $\ant(c),\con(c)\subseteq\Vsf^+$ define $\dec(c)$ to be
\[\bigwedge(\ant(c)\cap\Vsf)\cup\{q\;|\;q^{\neg}\in\con(c)\}\to\bigvee(\con(c)\cap\Vsf)\cup\{p\;|\;p^{\neg}\in\ant(c)\}\]
and for a conjunction of rules $\psi=c_1\wedge\ldots\wedge c_n$ set $\dec(\psi)=\dec(c_1)\wedge\ldots\wedge\dec(c_n)$.}
\textcolor{black}{Under these definitions, 
for every CNF $\varphi$ over $\Vsf$, we have $\dec(\enc(\varphi))\equiv\varphi$ because $\dec(\chi_{\text{setup}})\equiv\top$ and $\dec(\varphi^{\neg})\equiv\varphi$. Also, the following fact holds.}
\begin{lemma}\label{lem:newlemma}
\textcolor{black}{For every model $y\subseteq\Vsf^+$ such that $y=x^{\neg}$ for some $x\subseteq\Vsf$ and every CNF $\psi$ over $\Vsf^+$:}
\[\textcolor{black}{y\models\psi\qquad\textrm{iff}\qquad y\models\dec(\psi)^{\neg}}\]
\textcolor{black}{i.e. $\mod(\psi)\cap\mod(\chi_{\text{setup}})=\mod(\dec(\psi)^{\neg})\cap\mod(\chi_{\text{setup}})$.}
\end{lemma}
\begin{proof}
\textcolor{black}{Let $y=x^{\neg}\subseteq\Vsf^+$ and $\psi=c_1\wedge\ldots\wedge c_n$ be a CNF over $\Vsf^+$. Note that every rule of $\dec(\psi)^{\neg}$ is of the form $\dec(c_i)^{\neg}$ for some rule $c_i$ of $\psi$. It is not difficult to see that the every clause of the form $\dec(c)^{\neg}$ can be written as follows.}
\begin{align*}
    \textcolor{black}{\dec(c)^{\neg}=\bigwedge(\ant(c)\cup\{q^{\circ}\;|\;q\in\con(c)\})\to\bot}
\end{align*}

\textcolor{black}{Clearly $y\not\models\psi$ iff $y\not\models c_i$ for some $1\leq i\leq n$. In that case, for this $i$ we have $\ant(c_i)\subseteq y$ and $\con(c_i)\cap y=\emptyset$. But as $y=x^{\neg}$, this holds iff $\ant(c_i)\subseteq y$ and $\{q^{\circ}\;|\;q\in\con(c_i)\}\subseteq y$. But by the above definition of $\dec(c)^{\neg}$, we have $\ant(\dec(c_i)^{\neg})=\ant(c_i)\cup\{q^{\circ}\;|\;q\in\con(c)\}$ and $\con(\dec(c_i)^{\neg}))=\bot$, hence this holds iff $y\not\models\dec(c_i)^{\neg}$ for this $i$, which proves the claim.}
\end{proof}

\textcolor{black}{We want to show that learning the Horn envelope $\env(\enc(\varphi))$ suffices to learn $\varphi$, because 
$\varphi^{\neg}$ consists only of Horn rules over the extended signature $\Vsf^+$ and somehow captures all information about the original CNF $\varphi$. We observed above that $\dec(\enc(\varphi))\equiv\varphi$ but here we show the stronger claim that in fact $\dec(\env(\enc(\varphi)))\equiv\varphi$.}

\begin{lemma}\label{lem:envofenc}
The Horn envelope $\env(\enc(\varphi))$ is logically equivalent to $\Phi:=\varphi^{\neg}\cup\{(p\wedge p^{\neg})\to\bot\;|\;p\in\Vsf\}\cup\{\bigwedge\ant(c^{\neg})\setminus\{p\}\to p^{\circ}\;|\;c^{\neg}\in\varphi^{\neg},p\in\ant(c^{\neg})\}$\textcolor{black}{, and hence $\dec(\env(\enc(\varphi)))\equiv\varphi$ for any CNF $\varphi$ over $\Vsf$.}
\end{lemma}
\begin{proof}
We first show that $\env(\enc(\varphi))$ entails $\Phi$. We know that $\env(\enc(\varphi))\equiv\{h\;\text{Horn rule}\;|\;\enc(\varphi)\models h\}$.
So as $\varphi^{\neg}\cup\{(p\wedge p^{\neg})\to\bot\;|\;p\in\Vsf\}$ is a subset of $\enc(\varphi)$, these rules are entailed by $\enc(\varphi)$. Now take any rule  $\Phi$ of the form $\bigwedge(\ant(c^{\neg})\setminus\{p\})\to p^{\circ}$, where $c^{\neg}\in\varphi^{\neg}$ and $p\in\ant(c^{\neg})$. 
Recall that every Horn rule in $\varphi^{\neg}$ has an empty consequent because all positive literals have been substituted out. 
Hence by the valid rule of resolution:
\[(\bigwedge (P\cup\{p\})\to\bigvee Q)\wedge(\bigwedge P'\to\bigvee (Q'\cup\{p\}))\models\bigwedge P\cup P'\to\bigvee Q\cup Q'\]
we get:
\[c^{\neg}\wedge( p\vee p^{\neg})=(\bigwedge\ant(c^{\neg})\to\bot)\wedge (p\vee p^{\neg})\models\bigwedge(\ant(c^{\neg})\setminus\{p\})\to p^{\circ}.\]

For the other direction, we want to show that $\Phi\models\env(\enc(\varphi))$, i.e. $\mod(\Phi)\subseteq\mod(\env(\enc(\varphi)))=\closure(\mod(\enc(\varphi)))$. So let $x\models\Phi$. We need to show that $x$ is in the closure of $\mod(\enc(\varphi))$. By Remark \ref{rem:closure}, this is equivalent to checking whether $x=\bigcap\{e\in\mod(\enc(\varphi))\;|\;x\subseteq e\}$.
We know that for no $p\in\Vsf$ both $p$ and $p^{\neg}$ are in $x$, and we know that $x\models\varphi^{\neg}$. 

For each $p\in\Vsf$ such that both $p,p^{\neg}\not\in x$, we know that $x\cup\{p\}\models\bigwedge_{p\in\Vsf}(p\wedge p^{\neg})\to\bot$ still. If $x\cup\{p\}\not\models\varphi^{\neg}$ then, since $\varphi^{\neg}$ consists only of Horn rules of the form $\bigwedge\ant(h)\to\bot$, it must be that $\varphi^{\neg}\models\bigwedge(x\cup\{p\})\to\bot$. So since $\varphi^{\neg}$ consists only of negative Horn rules, there must be some rule $\bigwedge y\to\bot\in\varphi^{\neg}$ with $y\subseteq x\cup\{p\}$.

\textcolor{black}{It cannot be that $y\subseteq x$ otherwise $x\not\models\bigwedge y\to\bot$ while we assumed that $x\models\Phi$ and $\varphi^{\neg}$ is contained in $\Phi$. But then it must be that $p\in y$ and hence $\bigwedge y\setminus\{p\}\to p^{\circ}\in\Phi$, which contradicts our assumption that $x\models\Phi$ and $y\subseteq x\cup\{p\}$. Thus it must be that $x\cup\{p\}\models\varphi^{\neg}$ and a similar argument shows that $x\cup\{p^{\neg}\}\models\varphi^{\neg}$. But now $x=(x\cup\{p\})\cap(x\cup\{p^{\circ}\})$ which is in $\closure(\mod(\enc(\varphi)))=\mod(\env(\enc(\varphi)))$. Finally,  we have that  
$\dec(c^{\neg})\equiv c$, $\dec(p\wedge p^{\neg}\to\bot)\equiv\top$ and 
$\dec(\bigwedge(\ant(c^{\neg})\setminus\{p\})\to p^{\circ})\equiv\dec(c^{\neg})\equiv c$,
so $\dec(\env(\enc(\varphi)))\equiv\varphi$.}
\end{proof}

\begin{theorem}\label{thm:hardness}
Learning Horn envelopes (in polynomial time) with membership and equivalence queries is at least as hard as learning CNF (in polynomial time) with membership and equivalence queries.
\end{theorem}
\begin{proof}
\textcolor{black}{Take any CNF $\varphi$ over $\Vsf$ and consider its encoding $\enc(\varphi)$, which is a CNF over $\Vsf\cup\Vsf^{\neg}$ such that $\dec(\env(\enc(\varphi)))\equiv\varphi$ by Lemma \ref{lem:envofenc}.} Suppose there is an algorithm $\mathcal{A}$ that learns the Horn envelope in polynomial time. We will use $\mathcal{A}$ to learn a CNF representation of $\varphi$, going back and forth between the two settings with the $\enc(\cdot)$ and $\dec(\cdot)$ mappings. \textcolor{black}{Because we have that $\dec(\env(\enc(\varphi)))\equiv\varphi$, it suffices to show that we can correctly answer the oracle queries posed by $\mathcal{A}$, using our knowledge of the encoding and our two CNF oracles $\MQ_{\varphi}$ and $\EQ_{\varphi}$. For then $\mathcal{A}$ will terminate in polynomial time by assumption and output some representation $\psi$ of $\env(\enc(\varphi))$, whose decoding $\dec(\psi)$ would then be a representation of the original CNF $\varphi$.}

When $\mathcal{A}$ asks a membership query $\MQ_{\enc(\varphi)}(y)$ for some $y\subseteq\Vsf^+$, if $y\ne x^{\neg}$ for some $x\subseteq\Vsf$ then answer the membership query with ``no''. Else, for the unique $x\subseteq\Vsf$ such that $x^{\neg}=y$ ask the membership query $\MQ_{\varphi}(x)$ and return the answer to $\mathcal{A}$. When $\mathcal{A}$ asks an equivalence query $\EQ^{\textup{Horn}}_{\enc(\varphi)}(\psi)$, ask the equivalence query $\EQ_{\varphi}(\dec(\psi))$. If the CNF oracle answers ``yes'' to the latter query then $\dec(\psi)$ is a CNF representation of the original target CNF $\varphi$ and we are done. Otherwise the oracle returns a counterexample $x\in\mod(\varphi)\oplus\mod(\dec(\psi))$. It follows that $x^{\neg}\in\mod(\varphi^{\neg})\oplus\mod(\dec(\psi)^{\neg})$, \textcolor{black}{whence by Lemma \ref{lem:newlemma} we know that $x^{\neg}\in \mod(\varphi^{\neg})\oplus\mod(\psi)$. Also, since $\enc(\varphi)=\varphi^{\neg}\wedge\chi_{\text{setup}}$ and $\mod(\chi_{\text{setup}})=\{x^{\neg}\subseteq\Vsf^+\:|\;x\subseteq\Vsf\}$ we get that $x^{\neg}\in\mod(\enc(\varphi))\oplus\mod(\psi)$.} Hence ``no'' with the (necessarily Horn) counterexample $x^{\neg}$ is a valid answer to the equivalence query $\EQ^{\textup{Horn}}_{\enc(\varphi)}(\psi)$. It follows that learning Horn envelopes (in polynomial time) is as hard as learning CNFs (in polynomial time).
\end{proof}

\textcolor{black}{Consider the following variation of the above reduction: map any CNF $\varphi$ over $\Vsf$ to $\varphi^{\neg}$ (which is a Horn theory over $\Vsf^+$) and map every $x\subseteq\Vsf$ to $x^{\neg}$. This satisfies the condition of being a reduction among exact learning problems with equivalence queries \cite{DBLP:conf/aaai/KonevOW16, DBLP:journals/toct/HermoO20}, because the following holds.}
\[\textcolor{black}{x\models\varphi\qquad\textrm{iff}\qquad x^{\neg}\models\varphi^{\neg}}\]
\textcolor{black}{This shows that learning Horn formulas with only equivalence queries (in polynomial time) is as hard as CNF with only equivalence queries (in polynomial time). In fact it has been long known in the literature that Horn formulas are not learnable in polynomial time \cite{NegResforEQAngluin}. However it is \emph{not} a reduction under both membership and equivalence queries, because there is no way to simulate the answer to membership queries $\MQ_{\varphi^{\neg}}(y)$ for models $y\subseteq\Vsf^+$ for which there is no $x\subseteq\Vsf$ such that $x^{\neg}=y$. This makes sense because Horn formulas are known to be exact learnable with both queries in polynomial time \cite{Horn}.}

\section{Learning from Neural Networks}\label{sec:obstables}

In this section, we discuss in more details how we address the 
obstacles  mentioned in the Introduction to apply exact learning algorithms for extracting knowledge from trained neural networks. 
We start discussing the \textbf{second obstacle}.
Intuitively, we pose our queries to the neural network, thus \emph{viewing the neural network as the oracle}. To do this, one has to define a Boolean function from a trained neural network.
%
In this work, we create the lookup table presented in the appendix (Table~\ref{tab:attributes}) to make the conversion between Boolean values and expressions in natural language given to the language model. 
Discrete-valued attributes such as ``occupation'' with 11 values (including ``unknown occupation'') are encoded by 10 fresh variables intuitively representing propositions such as ``the occupation is mathematician''\footnote{For continuously valued attributes, one way to binarize is by chopping the continuous interval up into discretely many intervals and creating a fresh variable for each such interval.}.
We illustrate the conversion in Example~\ref{ex:conversion}.

\begin{example}\label{ex:conversion}
For the conversion, we use the lookup table and a template sentence. Suppose the template sentence is ``\textless mask\textgreater  was born [year] in [continent] and is a [occupation].''. Then, given a Boolean example
\[[0, 0, 0, 0, 1, 0, 1, 0, 0, 0, 0, 0, 0, 0, 0, 0, 1, 0, 0, 0, 0, 0, 0, 0, 1, 0]\]
the first $5$ positions represent the time period, the following $9$ positions are for the continent and then there are $10$ positions for the occupation. The last $2$ positions represent the gender, which is the true label.
The given example then translates to the sentence ``\textless mask\textgreater  was born after 1970 in Africa and is a dancer.'' with the true label ``female'' meaning the masked token should be filled with ``She''.
\end{example}
Once such correspondence with a Boolean function has been defined, one can use the neural network to answer oracle queries. Clearly, membership queries can be easily simulated by running the neural network on an example and check the classification. However, the \textbf{first obstacle} mentioned in the Introduction 
is that an equivalence oracle is hard to simulate in practice because it requires checking whether two formulas are equivalent and return a counterexample if this is not the case. 
In absence of an explicit representation for the Boolean function defined by the neural network, the only (foreseeable) way of checking equivalence w.r.t. the Boolean function defined by the neural network is to check all examples for agreement, which is an exponential task. 

Hence, we use the standard technique of simulating equivalence queries by random sampling \cite{angluinqueries}. That is, every time Algorithm \ref{alg:learn} asks an equivalence query $\EQ^{\textup{Horn}}_{\varphi}(\psi)$ we randomly generate a batch of examples and check whether the hypothesis $\psi$ classifies an example from this batch differently than $\env(\varphi)$, given the labels of $\varphi$ for this batch. It may happen that an example $x$ is labelled negatively yet it satisfies the envelope (if $x$ is a non-Horn negative example for $\varphi$). We only know that the classification of $x$ by $\varphi$ is different from $\env(\varphi)$ if we have observed a number of positive examples whose intersection is $x$; that is, if we have the data to prove that $x$ is non-Horn. Thus, the interpretation of the classification we receive from the oracle dynamically changes in response to the positive examples we receive. In other words, we learn what the real target $\env(\varphi)$ (that is, where it differs from the underlying formula $\varphi$) is whilst approximating it.

Regarding the \textbf{third obstacle}, it is unlikely that a Boolean function defined by a neural network (suitably binarized) defines exactly a Horn formula. This is because neural networks tend not to be \emph{rule-like} while a Horn formula is exactly a set of rules. This motivates studying the problem of learning Horn envelopes of arbitrary Boolean functions. However, we see that we can neither  use  the deterministic algorithm by Dechter and Pearl~\cite{DECHTER1992237} nor the probabilistic one by Kearns, Selman, and Kautz~\cite{KAUTZ1995129} because both assume access to a complete description of the positive examples (or its so-called ``characteristic models'' \cite{KAUTZ1995129}, whose intersection-closure is the set of all positive examples) which is an unrealistic assumption. 



While  Corollary~\ref{cor:polyifpolynonHorn} may seem like a weak statement, it is of practical interest. This is because real-world data tends to be \emph{sparse}. That is, $|\Vsf|$ tends to be big while $|\varphi|$ tends to be small~\cite{Borchmann_2020}. Similar remarks are made in footnote 4 of \cite{KAUTZ1995129}. We quote the following passage from \cite{DECHTER1992237} (adapted to fit our notation, where $M$ is a set of models):
\begin{quote}
``If $\closure(M)$ is substantially larger than $M$, we know that any Horn approximation is bound to be very poor. It is only when $|\closure(M)\setminus M|$ is a fraction of $|M|$ that Horn formulas can offer a reasonable approximation to $M$, and it is precisely in those cases that we can find a tightest Horn approximation in reasonable time. This suggest a strategy of focusing the development of Horn approximations in only those cases  that can benefit from such approximations.''
\end{quote}

In the next section we present our experimental results using a modified version of Angluin's algorithm for learning Horn formulas where queries are converted into natural language and posed to language models.

\section{Experiments}\label{sec:experiments}
We describe an experiment performed on different language models (LM) that we used as an oracle. 
In more details, we modify Angluin's Horn algorithm to make it applicable
to extract Horn expressions from the BERT-based language models~\cite{devlin2019}, including   RoBERTa-base and RoBERTa-large
~\cite{liuetal2019}. All models are 
accessed via the API of huggingface.\footnote{https://huggingface.co/}
%
 Our goal with this experiment is to showcase the applicability of the Horn algorithm for probing LMs and find out if occupation is generally more often linked to gender\footnote{We consider the (binary) male and female genders as the used data is based on this format.}  than  to other attributes.
For this comparison we use nationality and birth year as they are, next to gender, a defining attribute of every person. They represent culture and age in a very simple form and are therefore likely to also be linked to a persons' occupation, as certain age groups or cultures are more likely to have one occupation over another.
As a sanity check, we also perform a simple probing with the same language models and setup.

We extract a dataset from wikidata\footnote{https://www.wikidata.org/} that consists of every entity with a given occupation and their birth year, nationality, and gender. The used occupations are shown to be the 60 most gender biased occupations for the BERT-base model~\cite{DBLP:journals/corr/abs-2110-05367}.
The nationality is represented in the format of continents, as this reduces the amount of possible values drastically. For the same reason, the birth years are summarized into 5 time periods instead of distinct years. The exact time periods are determined from the dataset. The time period boundaries are then evenly spaced over all the birth years of all the entities in the dataset. This gives a more fine-grained difference in the 1900s, whereas everything before 1875 is summarized into one. The exact values can be seen in the first half of the lookup table (table \ref{tab:attributes}).
Each entity from the dataset then makes one example for probing by filling its attributes into the template sentence ``\textless mask\textgreater  was born [year] in [continent] and is a [occupation].''
The probing is done by predicting the masked pronoun in each sentence  with the given language model. 

\begin{table}[ht]
    \centering
    \begin{tabular}{|c|c|c|}
        \hline
         attribute & extracted attribute & transformed attribute \\
         \hline
         name & Ernst Zierke & Ernst Zierke \\
         gender & male & male \\
         birth year & 1905 & between 1892 and 1934 \\
         nationality & Germany & Europe \\
         \hline
    \end{tabular}
    \caption[Example data point]{One data point   for the occupation ``nurse'', showing both the extracted attributes and their transformed version after dimensionality reduction.}
    \label{tab:exampledata}
\end{table}

The difference in the resulting probabilities for ``He'' and ``She'' is then used to calculate the gender bias on sentence $i$ \cite{DBLP:journals/corr/abs-2110-05367}:
\[ \text{Pronoun Prediction Bias Score} = \textcolor{black}{ppbs}_i = p(He) - p(She).\]
With $N_{occ}$ being the number of examples for occupation $occ$, the Pronoun Prediction Bias Score for occupation $occ$ is then the average over all examples with this occupation
\[ \textcolor{black}{ppbs}_{occ} = \frac{1}{N_{occ}} \sum_{i=1}^{N_{occ}} \textcolor{black}{ppbs}_{i}\]

\begin{figure*}
        \centering
        \includegraphics[width=\textwidth]{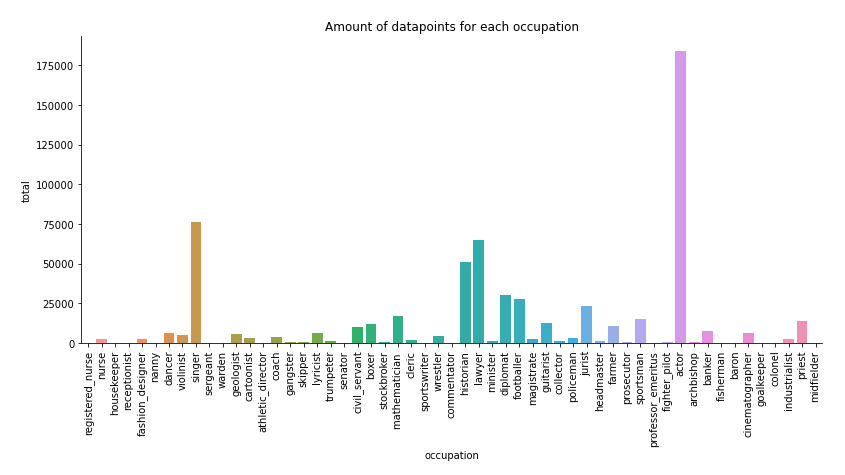}
        \caption{Amount of data points collected for each occupation from wikidata.} 
        \label{fig:dataset}
\end{figure*}
\begin{figure*}
        \centering
        \includegraphics[width=\textwidth]{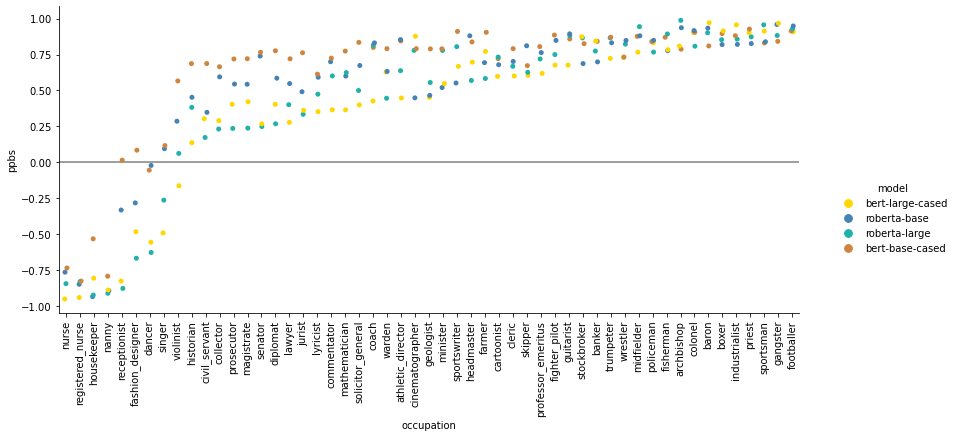}
        \caption{Pronoun Prediction Bias Score of 60 occupations based on the wikidata dataset for BERT and RoBERTa models.} 
        \label{fig:ppbs}
    \end{figure*}

Based on the results of the probing (Figure \ref{fig:ppbs}) and the frequency of the occupations in the data (Figure \ref{fig:dataset}), we chose 10 clearly biased occupations for the Horn algorithm. 
In our experiment, we developed a function that creates a sentence out of given attributes that are encoded in the variables of each interpretation. In the context of the task, an interpretation corresponds to an entity with certain attributes.

Each attribute is one-hot encoded into a vector with at most one $1$ and all of the attribute vectors together are one interpretation (=entity). In particular, the $4$ attributes are: period of time ($5$ features), nationality (as a continent, $9$ features), occupation ($10$ features), and gender ($2$ features). The attributes are handled in the same way as in the probing experiment. With this, each interpretation can be translated into natural language attributes using a lookup table (Table \ref{tab:attributes}), which can then be filled into the template sentence.

\textcolor{black}{We assume
that exactly one of the genders dominates the probability value. As an example, we   look at the data point given in Table \ref{tab:exampledata} and the probing results of this data point in Table~\ref{tab:exampleresults}. The precise attributes given in the first column are transformed according to the dimensionality reduction method mentioned above, as shown in the second column. The attributes are then filled into the template sentence ``\texttt{\textless mask\textgreater  was born [birth year] in [nationality] and is a/an [occupation].}''.} 
\textcolor{black}{In our example, the  input for the RoBERTa models is  ``\texttt{\textless mask\textgreater \ was born between 1892 and 1934 in Europe and is a nurse.}''\footnote{The mask token changes according to
the model that is used. For the RoBERTa models, the mask token is \texttt{\textless mask\textgreater} whereas for the BERT models, it is \texttt{[MASK]}. We have accordingly changed the sentences when probing the models. }.}

\begin{table}[ht]
    \centering
    \begin{tabular}{|p{5cm}|c|c|c|c|}
        \hline
        input sentence & label & $p(He)$ & $p(She)$ & $p(They)$ \\
        \hline
        \textless mask\textgreater \  was born between 1892 and 1934 in Europe and is a nurse. & ``He'' & $0.070167$ & $0.786434$ & $0.0$ \\
         \hline
    \end{tabular}
    \caption[Probing results for example data point]{Template-based probing results for the  data point given in Table \ref{tab:exampledata}.}
    \label{tab:exampleresults}
\end{table}

\textcolor{black}{Every data point in the dataset is handled this way and used to query the language models to fill the masked token. }
\textcolor{black}{The resulting probability distribution (of the top 5 results) and prediction of each language model is then saved. Table \ref{tab:exampleresults} shows the results for the given example for illustration purposes. }

For a membership query, the language model predicts the gender of the given entity by predicting the masked pronoun in the sentence. We compare this prediction with the given gender and return whether they match or not as the result of the query.
We generate the samples for an equivalence query as random feature vectors, given that each attribute can have at most one $1$ in it. 
The number of equivalence queries simulated by the Horn algorithm was limited to $50$, $100$, and $200$ for different experiments. 
For each language model we conducted $10$ iterations of each experiment.

The results from the extractions of each language model expose biases in all of them (Tables \ref{tab:BERT100}, \ref{tab:roBERTa100}, \ref{tab:BERT200} and \ref{tab:roBERTa200}). We consider the rules that were extracted in at least $7$ iterations as the most relevant and reliable ones. With 100 equivalence queries, the relevant rules for each language model link gender and occupation without taking other attributes into account (with one exception). 
The other attributes almost only appear in less relevant rules that have been extracted in $3$ or less iterations. The only exception to that is the rule \text{$ \sf{singer} \land \sf{male}  \rightarrow \sf{before\_1875} $} extracted by RoBERTa-large in 100 equivalence queries, which was extracted in 7 out of 10 iterations.
The same rules are appearing in 10 out of 10 iterations with 200 equivalence queries for all models. This confirms that those rules are the most relevant ones. It also shows that the number of equivalence queries that is used as a maximum is important for the kind of rules that are extracted and how reliable they are. 

\begin{table}
    \centering
    \begin{tabular}{|c|}
        \hline
        $\text{priest} \land \text{female}  \rightarrow \bot$ \\
        $\text{nurse} \land \text{male}  \rightarrow \bot$ \\
        $\text{mathematician} \land \text{female}  \rightarrow \bot$ \\
        $\text{footballer} \land \text{female}  \rightarrow \bot$ \\
        $\text{banker} \land \text{female}  \rightarrow \bot$ \\
       \hline
    \end{tabular}
    \caption{Intersection of rules from all language models (10/10 with 200 EQs).}
    \label{tab:intersection}
\end{table}

Recall that a rule that has $\bot$ in the consequent means that the antecedent does not happen. In addition, we consider gender to be exclusively binary\footnote{This experiment is done under the assumption of exclusive binary gender. We acknowledge that this is not the reality.} and therefore it also holds that $\lnot \sf female \leftrightarrow male$. In other words, we extracted rules revealing certain stereotypes e.g. stating that ``women are not football players'' and ``nurses are women''. All extracted stereotypes of this kind match with the results from the probing experiment.
It is also important to note that out of all rules extracted, the base models (RoBERTa-base and BERT-base) only relate the male gender with ``nurse'' without relating the other female perceived occupations ``fashion-designer'', ``dancer'', and ``singer'' as well. On the other side, the female gender is related to all male perceived occupations, even those that are more lightly biased. This also shows that bias towards females is more present than bias towards males. The latter is only extracted in the strongest case of ``nurse''.

This experiment took approximately 1, 3, and 13 hours per iteration with 50, 100, and 200 equivalence queries respectively for the base models on a PowerEdge R7525 Server.
For the large models, one iteration took approximately 2, 5, and 15 hours for 50, 100, and 200 equivalence queries respectively on a PowerEdge R7525 Server (Table \ref{tab:runtimes}) 
Although we could see an improvement in the quality of rules we extracted with 200 equivalence queries, the runtime is also significantly higher. In our experiments, 100 equivalence queries were sufficient to extract the same rules in 70\% of experiments. There is a trade off between the runtime and the quality of the extracted rules that favors multiple rounds of Horn with 100 equivalence queries over less rounds with 200 equivalence queries.

\begin{table}[h]
    \centering
    \begin{tabular}{|c|c|c|c|c|}
        \hline
        \# EQs & BERT-base & BERT-large & RoBERTa-base & RoBERTa-large \\
        \hline
        50    &  $71.74$ & $130.01$ &  $69.76$ & $129.22$ \\
        100   & $193.74$ & $303.96$ & $184.82$ & $308.73$\\
        200   & $722.55$ & $899.13$ & $771.97$ & $943.26$\\
        \hline
    \end{tabular}
    \caption{Average run time for one experiment iteration [in minutes].}
    \label{tab:runtimes}
\end{table}

\section{Conclusion}\label{sec:conclusion}

\textcolor{black}{We studied the problem of exact learning Horn envelopes using membership queries and equivalence queries \cite{angluinqueries}. This theoretical problem was motivated from the objective to apply exact learning algorithms to extract knowledge from trained neural networks, where we use random sampling to simulate the equivalence queries and the network as a membership oracle. In particular, our work takes into account the fact  that trained neural networks may not be   Horn.}

\textcolor{black}{We presented an algorithm that learns Horn envelopes in exponential time in general, and in polynomial time (in the size of the smallest equivalent CNF representation of the target) if the target is ``close to Horn'' in the sense that there are only polynomially many non-Horn negative examples for it. We noted that this algorithm also learns under the same conditions as above when applied to a CNF equivalence oracle instead of a Horn equivalence oracle. Thus we also showed that, \emph{within} the class of all CNFs, the ``almost-Horn'' CNFs are learnable in polynomial time. We also showed that learning Horn envelopes in polynomial time is as hard as learning arbitrary CNF in polynomial time (a problem  known to be hard \cite{FELDMAN200913}).}

\textcolor{black}{We also performed   experiments where we adapt Angluin's classical algorithm for exact learning Horn theories to make it applicable to learn from   masked  language models. We performed experiments on pre-trained  language models and extracted rules exposing occupation-based gender biases in these models. 
While these results are not surprising given  the results of several authors when probing language models (see Subsection~\ref{sec:probing}) and existing gender biases in the society~\cite{bias}, our approach provides a
way  of exploring other potential correlations such as those related to time periods and location. }

\section*{Acknowledgements}
Koudijs and Ozaki are supported by
the NFR project ``Learning Description Logic Ontologies'',
grant number 316022, led by Ozaki. 
Touileb is supported by industry partners and the Research Council of Norway with funding to \textit{MediaFutures: Research Centre for Responsible Media Technology and Innovation}, through the centers for Research-based Innovation scheme, project number 309339.

\bibliographystyle{plain}
\bibliography{references.bib}

\appendix

\section{Non-termination of HORN when the target is non-Horn}\label{sec:appendix}
In this appendix, we give the pseudo-code of the classical algorithm and give an example run of a non-Horn target with a choice of counterexamples on which it does not terminate (see also~\cite{DBLP:journals/corr/abs-2204-00361}).

\SetKwInput{KwData}{input}
\SetKwInput{KwResult}{output}
\begin{algorithm}\label{alg:Horn}
\SetAlgoVlined
\LinesNumbered
\caption{HORN}
\KwData{$\MQ_\varphi$ and $\EQ_\varphi$ oracles.}
\KwResult{A Horn formula $\Hmc$ such that $\Hmc\equiv\varphi$.}
\BlankLine
$\Hmc:=E^-:=\emptyset$\\
\While{$\EQ_\varphi(\Hmc)=(\textup{no},x)$ \label{lin:while}}
{  \eIf{$x\not\models \Hmc$}
        {\For{$h\in \Hmc$ \textup{such that} $x\not\models h$}   
        {$\Hmc:=\Hmc\setminus\{h\}$}
        }
        {   \eIf{$\exists e\in E^-\;\textup{s.t.}\;\MQ_\varphi(x \cap e)=\textup{no}\;\textup{and}\;x\cap e\subset e$ \label{lin:condition}}
            {let $e$ be the first such negative example \\
            replace $e\in E^-$ with $e\cap x$ \label{lin:replace}
            }
            {append $x$ to $E^-$\\}
            }
            $\Hmc:=\bigcup_{e\in E^-}\{\bigwedge e\to p\;|\;p\not\in e\}\cup\{\bigwedge e\to\bot\}$
            
        }
\textbf{return} $(\Hmc)$
\end{algorithm}

\begin{proposition}
Algorithm \ref{alg:Horn} (i.e. the classic Horn learning algorithm of \cite{Horn}) is not guaranteed to terminate when the target is non-Horn.
\end{proposition}
\begin{proof}
Let $\Vsf=\{a,b,c,d\}$ and let $\varphi=\{a\to\bot, b\vee c\}$ be the target. We have that  $\env(\varphi)=\{a\to\bot\}$ and hence e.g. $\{d\}$ is a non-Horn negative example for $\varphi$ because $\{d\}\models\env(\varphi)$ but $\{d\}\not\models b\vee c$. Now we describe an adversarial choice of counterexamples that forces Algorithm \ref{alg:Horn} to add the same metarule and then remove it to get an empty hypothesis $\Hmc$ over and over again. First, we return the non-Horn negative example $\{d\}$, whence the rules $d\to a,d\to b,d\to c$ and $d\to\bot$ are added to $\Hmc$. Then we return the positive counterexample $\{b,d\}\not\models d\to c,d\to\bot$, after which only the rule $d\to b$ is not removed $\Hmc$. Next, we give the positive counterexample $\{c,d\}\not\models d\to b$ after which also $d\to b$ is removed from $\Hmc$ so $\Hmc$ is empty. Therefore, if these three counterexamples are returned again and again in this order, Algorithm \ref{alg:Horn} does not terminate.\footnote{The optimized version of the classic Horn algorithm (called ``HORN1'' in \cite{Horn}) is not even well-defined for a non-Horn target. In this case we would end up with a metarule whose consequent is empty. This is not the same as the consequent of a metarule being $\bot$; in fact this has no meaning.} 
\end{proof}

\section{Experiment Results and Tables}

\begin{table}[ht]
    \centering
    \begin{tabular}{|c|c|c|c|}
         \hline
         \# & BERT-base & \# & BERT-large \\
         \hline
            10 & $\text{nurse} \land \text{male}  \rightarrow \bot$ & 10 & $\text{nurse} \land \text{male}  \rightarrow \bot$\\
            10 & $\text{diplomat} \land \text{female}  \rightarrow \bot$ & 10 & $\text{diplomat} \land \text{female}  \rightarrow \bot$\\
            10 & $\text{mathematician} \land \text{female}  \rightarrow \bot$ & 10 & $\text{mathematician} \land \text{female}  \rightarrow \bot$\\
            10 & $\text{banker} \land \text{female}  \rightarrow \bot$ & 10 & $\text{banker} \land \text{female}  \rightarrow \bot$\\
            9 & $\text{footballer} \land \text{female}  \rightarrow \bot$ & 10 & $\text{footballer} \land \text{female}  \rightarrow \bot$\\
            9 & $\text{lawyer} \land \text{female}  \rightarrow \bot$ & & \\
            8 & $\text{priest} \land \text{female}  \rightarrow \bot$ & 10 & $\text{priest} \land \text{female}  \rightarrow \bot$\\
            & & 10 &  $\text{singer} \land \text{male}  \rightarrow \bot$ \\
            & & 10 & $\text{dancer} \land \text{male}  \rightarrow \bot$ \\
         \hline
    \end{tabular}
    \caption{
    Rules  extracted at least 7 out of 10 times   with  BERT models and 100 equivalence queries.
    }
    \label{tab:BERT100}
\end{table}

\begin{table}[ht]
    \centering
    \begin{tabular}{|c|c|c|c|}
    \hline
        \# & RoBERTa-base & \# & RoBERTa-large \\
         \hline
            10 & $\text{priest} \land \text{female}  \rightarrow \bot$ & 10 & $\text{priest} \land \text{female}  \rightarrow \bot$\\
            10 & $\text{nurse} \land \text{male}  \rightarrow \bot$ & 10 & $\text{nurse} \land \text{male}  \rightarrow \bot$\\
            10 & $\text{diplomat} \land \text{female}  \rightarrow \bot$ & & \\
            10 & $\text{mathematician} \land \text{female}  \rightarrow \bot$ & 10 & $\text{mathematician} \land \text{female}  \rightarrow \bot$\\
            9 & $\text{banker} \land \text{female}  \rightarrow \bot$ & 10 & $\text{banker} \land \text{female}  \rightarrow \bot$\\
            9 & $\text{footballer} \land \text{female}  \rightarrow \bot$ & 10 & $\text{footballer} \land \text{female}  \rightarrow \bot$\\
            8 & $\text{lawyer} \land \text{female}  \rightarrow \bot$ & 10 & $\text{lawyer} \land \text{female}  \rightarrow \bot$\\
            & & 10 & $\text{fashion\_designer} \land \text{male}  \rightarrow \bot$ \\
            & & 10 & $\text{dancer} \land \text{male}  \rightarrow \bot$ \\
            & & 7 &  $\text{singer} \land \text{male}  \rightarrow \text{before 1875} $ \\
         \hline
    \end{tabular}
    \caption{Rules  extracted at least 7 out of 10 times   with   RoBERTa models and 100 equivalence queries.}
    \label{tab:roBERTa100}
\end{table}

\begin{table}[ht]
    \centering
    \begin{tabular}{|c|c|c|c|}
         \hline
         \# & BERT-base & BERT-large \\
         \hline
            10 & $\text{lawyer} \land \text{female}  \rightarrow \bot$ & \\
            10 & $\text{diplomat} \land \text{female}  \rightarrow \bot$ &$\text{diplomat} \land \text{female}  \rightarrow \bot$ \\
            10 & $\text{priest} \land \text{female}  \rightarrow \bot$ &$\text{priest} \land \text{female}  \rightarrow \bot$ \\
            10 & $\text{nurse} \land \text{male}  \rightarrow \bot$ &$\text{nurse} \land \text{male}  \rightarrow \bot$ \\
            10 & $\text{mathematician} \land \text{female}  \rightarrow \bot$ & $\text{mathematician} \land \text{female}  \rightarrow \bot$ \\
            10 & $\text{footballer} \land \text{female}  \rightarrow \bot$ & $\text{footballer} \land \text{female}  \rightarrow \bot$ \\
            10 & $\text{banker} \land \text{female}  \rightarrow \bot$ & $\text{banker} \land \text{female}  \rightarrow \bot$ \\
            8 & $\text{dancer} \land \text{male} \land \text{South America}  \rightarrow \bot$ & \\
            8 & $\text{1875-1925} \land \text{fashion\_d} \land \text{female}  \rightarrow \bot$ & \\
            7 & $\text{dancer} \land \text{male} \land \text{Europe}  \rightarrow \bot$ & \\
            7 & $\text{fashion\_d} \land \text{1925-1951} \land \text{female}  \rightarrow \bot$ & \\
            7 & $\text{dancer} \land \text{male} \land \text{North America}  \rightarrow \bot$ & \\
            10 & & $\text{singer} \land \text{male}  \rightarrow \bot$ \\
            10 & & $\text{dancer} \land \text{male}  \rightarrow \bot$ \\
            10 & & $\text{lawyer} \land \text{female}  \rightarrow \text{Australia} $ \\
            8 & & $\text{fashion\_d} \land \text{male}  \rightarrow \text{Americas} $ \\
            8 & & $\text{fashion\_d} \land \text{male}  \rightarrow \text{before 1875} $ \\
         \hline 
    \end{tabular}
    \caption{
    Rules  extracted at least 7 out of 10 times   with  BERT models and 200 equivalence queries.
    }
    \label{tab:BERT200}
\end{table}

\begin{table}[ht]
    \centering
    \begin{tabular}{|c|c|c|}
    \hline
        \# & RoBERTa-base & RoBERTa-large \\
         \hline
        10 & $\text{lawyer} \land \text{female}  \rightarrow \bot$ & $\text{lawyer} \land \text{female}  \rightarrow \bot$\\
        10 & $\text{diplomat} \land \text{female}  \rightarrow \bot$ &\\
        10 & $\text{priest} \land \text{female}  \rightarrow \bot$ & $\text{priest} \land \text{female}  \rightarrow \bot$\\
        10 & $\text{nurse} \land \text{male}  \rightarrow \bot$ & $\text{nurse} \land \text{male}  \rightarrow \bot$\\
        10 & $\text{mathematician} \land \text{female}  \rightarrow \bot$ & $\text{mathematician} \land \text{female}  \rightarrow \bot$ \\
        10 & $\text{footballer} \land \text{female}  \rightarrow \bot$ & $\text{footballer} \land \text{female}  \rightarrow \bot$\\
        10 & $\text{banker} \land \text{female}  \rightarrow \bot$ & $\text{banker} \land \text{female}  \rightarrow \bot$\\
        10 & & $\text{dancer} \land \text{male}  \rightarrow \bot$ \\
        10 & & $\text{fashion\_d} \land \text{male}  \rightarrow \bot$ \\
        10 & & $\text{diplomat} \land \text{female}  \rightarrow \text{Oceania} $ \\
        9 & $\text{Australia} \land \text{fashion\_d} \land \text{male}  \rightarrow \bot$ &\\
        9 & $\text{Oceania} \land \text{dancer} \land \text{male}  \rightarrow \bot$ &\\
        8 & $\text{singer} \land \text{female} \land \text{1951-1970}  \rightarrow \bot$ &\\
        8 & $\text{dancer} \land \text{male} \land \text{Africa}  \rightarrow \bot$ &\\
        8 & & $\text{singer} \land \text{male} \land \text{1951-1970}  \rightarrow \bot$ \\
         \hline
    \end{tabular}
    \caption{Rules  extracted at least 7 out of 10 times   with   RoBERTa models and 200 equivalence queries.}
    \label{tab:roBERTa200}
\end{table}

\begin{table}[ht]
    \centering
    \begin{tabular}{|c|c|}
        \hline
        position & value \\
        \hline
        \multicolumn{2}{|c|}{time period}\\
        \hline
        0 & before 1875 \\
        1 & between 1975 and 1925 \\
        2 & between 1925 and 1951 \\
        3 & between 1951 and 1970 \\
        4 & after 1970 \\
        - & in an unknown time period \\
        \hline
        \multicolumn{2}{|c|}{continent}\\
        \hline
        5 & North America \\
        6 & Africa \\
        7 & Europe \\
        8 & Asia \\
        9 & South America \\
        10 & Oceania \\
        11 & Eurasia \\
        12 & Americas \\
        13 & Australia \\
        - & an unknown place \\
        \hline
        \multicolumn{2}{|c|}{occupation}\\
        \hline
        14 & fashion designer \\
        15 & nurse \\
        16 & dancer \\
        17 & priest \\
        18 & footballer \\
        19 & banker \\
        20 & singer \\
        21 & lawyer \\
        22 & mathematician \\
        23 & diplomat \\
        - & not known occupation \\  
        \hline
        \multicolumn{2}{|c|}{gender}\\
        \hline 
        24 & female \\
        25 & male \\
        \hline
    \end{tabular}
    \caption{Lookup table. For each attribute, at most one of the positions can be chosen (setting its value to 1). If none is chosen, we use   the value in ``-''.}
    \label{tab:attributes}
\end{table}

\end{document}